\documentclass[final,12pt]{colt2025} %

\usepackage{amsfonts}
\usepackage{mathtools}

\usepackage{latexsym}
\usepackage{relsize}
\usepackage{mathabx}
\usepackage{comment}
\usepackage{wrapfig}
\usepackage{tikz}

\usepackage{hyperref}       %

\usepackage{xcolor}
\hypersetup{
    colorlinks,
    linkcolor={blue!80!black},
    citecolor={blue!80!black},
    urlcolor={blue!80!black}
} 

 \usepackage{algorithm}
\usepackage{algorithmic}
\usepackage{braket}

\usepackage{booktabs}       %
 \usepackage[shortlabels]{enumitem}
\usepackage[bbgreekl]{mathbbol}

\DeclareSymbolFontAlphabet{\mathbb}{AMSb}
\DeclareSymbolFontAlphabet{\mathbbl}{bbold}

\usepackage{thmtools,thm-restate}

\usepackage[capitalise]{cleveref}
\usepackage{dsfont}

\newcommand{\T}{\mathcal{T}}

\newcommand{\KL}{\mathrm{KL}}
\newcommand{\regret}{\mathsc{Regret}}

\newcommand{\depth}{\mathrm{QC}}

\def\X{{\mathcal X}}

\def\Y{{\mathcal Y}}

\def\R{{\mathcal R}}

\def\E{{\mathbb E}}

\newcommand{\A}{\mathcal{A}}

\newcommand{\F}{\mathcal{F}}

\newcommand{\Alg}{\mathrm{Alg}}

\def\regret{\mbox{{Regret}}}

\newcommand{\ignore}[1]{}

\DeclareMathAlphabet{\mathbfsf}{\encodingdefault}{\sfdefault}{bx}{n}

\DeclareMathOperator*{\argmin}{arg\,min}
\DeclareMathOperator*{\argmax}{arg\,max}

\let\Pr\relax
\DeclareMathOperator{\Pr}{\mathbb{P}}

\renewcommand{\leq}{~\le~}
\renewcommand{\geq}{~\ge~}

\usepackage{bbm}

\title[On the Hardness of  Bandit Learning]{On the Hardness of Bandit Learning}
\usepackage{times}

\coltauthor{%
 \Name{Nataly Brukhim}\thanks{Equal contribution.} \Email{nbrukhim@princeton.edu}\\
 \addr Institute for Advanced Study
 \AND
 \Name{Aldo Pacchiano}\footnotemark[1] \Email{pacchian@bu.edu}\\
 \addr Boston University \& Broad Institute of MIT and Harvard%
\AND
 \Name{Miroslav Dudik} \Email{mdudik@microsoft.com}\\
 \addr Microsoft Research%
\AND
 \Name{Robert Schapire} \Email{schapire@microsoft.com}\\
 \addr Microsoft Research%
}

\begin{document}

\maketitle

\begin{abstract}%

We study the task of bandit learning, also known as best-arm identification, under the assumption that the true reward function $f$ belongs to a known,  but arbitrary, function class $\F$.  
While many instances of this problem are well understood, we seek a general theory of bandit learnability, akin to the PAC framework for classification. Our investigation is guided by the following two fundamental questions: (1) \textit{which} classes $\F$ are learnable,  and (2) \textit{how} they are learnable. For example, in the case of binary PAC classification,  learnability is fully determined by a combinatorial dimension, namely, the VC dimension, and can be attained via a simple algorithmic principle, namely, empirical risk minimization (ERM).

In contrast to classical learning-theoretic results, our findings reveal 
fundamental limitations of learning in structured bandits, offering new insights into the boundaries of bandit learnability. First, for the question of ``\textit{which}",  we show that the paradigm of identifying the learnable classes via a {dimension-like} quantity fails for bandit learning. We give a simple proof demonstrating that no combinatorial dimension can characterize bandit learnability, even in finite classes, following a standard definition of dimension introduced by \cite{ben2019learnability}. 

For the question of ``\textit{how}", we prove a computational hardness result:   we construct a reward function class for which at most two queries are
needed to find the optimal action, yet no algorithm can do so in polynomial time, unless $\mathrm{RP}=\mathrm{NP}$. Perhaps surprisingly, we also prove that this class admits efficient algorithms for standard
(albeit possibly computationally hard) algorithmic operations often considered in learning theory, such as an ERM. Therefore, this implies  that computational hardness is in this case inherent to the task of bandit learning.

Beyond these results, we investigate additional themes such as learning under noise, trade-offs between noise models, and the relationship between query complexity and regret minimization.

\end{abstract}

\section{Introduction}

In statistical learning theory, the \emph{probably approximately correct} (PAC) framework \citep{Valiant84} is central to understanding binary classification learnability. A key result shows that PAC learnability is fully determined by the VC dimension \citep{vapnik:74, blumer:89}, elegantly linking learnability and sample complexity. Similar characterizations exist for widely diverse variants of statistical and online learning (e.g., \citealp{bartlett1994fat, littlestone1988learning}). 
The appeal of combinatorial characterizations is often in their {simplicity}, reducing learnability to a single parameter, and offering useful insights into algorithm design and problem structure.

In contrast, a similarly tight characterization of \textit{bandit learning}, and in particular of the problem known as \emph{best-arm identification}, is still lacking.
In this setting, there is a set of \textit{actions} (or arms) $\A$, and an unknown \textit{reward function} $f^*: \A \rightarrow [0,1]$. A learner repeatedly queries actions $a \in \A$  and observes their corresponding reward, a random variable with mean $f^*(a)$. The aim of a learning algorithm is to identify a near-optimal action using as few queries as possible. Analogous to classic learning settings, one may assume the rewards are \textit{realizable} by a known, but arbitrary, class of reward functions $\F \subseteq [0,1]^\A$. A key focus of study in this context is the optimal \textit{query complexity} associated with a given class.

The pursuit of VC-dimension-like parameters for bandit learning has drawn considerable attention \citep{amin2011bandits,
russo2013eluder, foster2021statistical,
brukhim2023unified, foster2023tight,  hanneke2024complete}. However, existing parameters are often non-combinatorial in nature, rather complex, and in the general case exhibit substantial gaps between upper and lower bounds (see \Cref{subsec:related} for further discussion). We start our investigation by asking whether there exists a combinatorial characterization of bandit learning. Somewhat disappointingly, we prove that no such characterization of bandit learnability exists. Specifically, we use the definition of a combinatorial dimension introduced by \cite{ben2019learnability} that encompasses all standard notions of dimension in both statistical and online learning.  Using a rather simple argument, we demonstrate that no such dimension can universally characterize bandit learnability, even for finite classes.

We then shift our focus to exploring algorithmic approaches to the problem.  Specifically, we examine reward function classes with \textit{small} optimal query complexity and seek a general algorithmic principle that achieves it, guided by the question:
\begin{center}
    When is a class $\F$ of a \textit{bounded} query complexity, \textit{efficiently} bandit-learnable?
\end{center}
There are several algorithmic oracle assumptions commonly considered in the context of computational efficiency. For example, in  statistical learning theory, the gold standard is the simple \emph{empirical risk minimization} (ERM) principle which determines that it suffices to find any function in the class that is consistent with past observations.  
In interactive settings, an estimation algorithm is often used both to find a consistent function and to produce future predictions (see, e.g., \citealp{foster2023tight, brukhim2023unified}).  
One might assume that a class which admits efficient algorithms as above might also be  efficiently bandit-learnable. Interestingly, 
we prove a hardness result showing that is not the case. Specifically, we construct a reward function class for which at most \textit{two  queries} are needed to find the optimal action, yet no algorithm can do so in polynomial time, unless $RP = NP$. Moreover, we prove that this class admits efficient algorithms to the aforementioned tasks, demonstrating that the hardness is inherent to the bandit setting.

An important aspect of bandit learnability is the noise model being considered. In the absence of noise, learning is less constrained and  is therefore simpler. However, it is also more brittle, as it relies heavily on the precise function values that define the structure of the class $\F$. In contrast, under sufficiently noisy conditions, bandit learnability exhibits a form of robustness, allowing it to be characterized by simple parameters and algorithms, as shown in recent work by \cite{hanneke2024complete}. 
However, while some works \citep{hanneke2023bandit, amin2011bandits} focused on the noise-free regime, \cite{hanneke2024complete} considered a highly complex family of distributions of arbitrary noise, leaving intermediate noise regimes largely unaddressed (see further discussion in \Cref{subsec:related}).

In this work, we partially address this gap and examine the effect of noise on the query complexity of bandit learning. We focus on a Gaussian noise model and study the relationship between the noise variance and the query complexity. For instance, we show that certain function classes have a query complexity of $1$ when $\sigma=0$ but become unlearnable (i.e., infinite query complexity) when $\sigma=1$. Moreover, we identify an upper bound $\bar{\sigma}$ on $\sigma$ such that, for any function class, the query complexity for any $\sigma \le \bar{\sigma}$ is upper bounded by the query complexity for $\sigma=0$. This observation implies that  the query complexity in the low-noise regime can be captured by the noise-free setting.

Additionally, we prove that for a specific family of function classes, there exist class-dependent thresholds for $\sigma$, that separate distinct learning regimes.
Above a certain noise level, the query complexity is governed by a simple parameter $\gamma$ known as the \emph{generalized maximin volume}, introduced by \cite{hanneke2024complete}. Below a different threshold the query complexity is $1$, exhibiting a large gap from $\gamma$. Understanding the broader interplay between noise variance and query complexity across arbitrary function classes remains an open and interesting direction for future research.

Finally, we examine an alternative notion of learnability in bandits via  the lens of regret minimization and study its relationship with query complexity of best-arm identification.
Specifically, we prove that any algorithm which achieves the optimal query complexity $d$, must also incur regret that is linear in $d$, and is \emph{not} regret-optimal for time horizon $T = O(d)$. This result establishes that no single algorithm can simultaneously achieve both optimal query complexity and optimal regret.

\subsection{Related work}\label{subsec:related}

The PAC framework and related combinatorial characterizations have played a crucial role in providing quantitative insights into learnability across statistical learning theory. However, bandit learning, particularly best-arm identification (BAI), lacks a unifying framework and remains largely a collection of case-specific analyses (see, e.g., \citealp{bubeck2012regret}, and references within). 
 Moreover, most prior BAI work (e.g., \citealp{garivier2016optimal, kaufmann2016complexity})  assume that the mean rewards lie in some fixed bounded product space, e.g.,  $\mathcal{F} = [0,1]^K$, and so pulling one of the $K$ arms provides no information about others. In contrast, the focus of this work is the setting in which observations can possibly reveal additional information, based on the structure of the class $\mathcal{F} \subsetneq [0,1]^K$.

Indeed, the approach of studying the structure of the class itself has gained attention in recent years \citep{foster2021statistical, foster2023tight, hanneke2023bandit, hanneke2024complete}.
A notable proposed parameter for capturing interactive decision making is the decision-estimation coefficient (DEC) \citep{foster2021statistical, foster2023tight}. However, it suffers from arbitrarily large gaps between upper and lower bounds \citep{foster2023tight} and fails to characterize learnability in stochastic bandits (see \citealp{hanneke2024complete}). 

More recently, \cite{hanneke2024complete} introduced a characterization for stochastic bandits with \emph{arbitrary} noise, but it exhibits an exponential gap between upper and lower bounds and does not seamlessly extend to standard noise models, e.g., Gaussian noise. In \Cref{sec:noisy_vs_det}, we further analyze their generalized maximin volume parameter, showing that under moderate-variance Gaussian noise, it can diverge arbitrarily from the optimal query complexity.

Finally, we establish that no combinatorial dimension fully characterizes bandit learnability. While \cite{hanneke2023bandit} demonstrated a related result using complex set-theoretic arguments, their proof relies on the cardinality of the continuum and does not directly address combinatorial dimensions. In contrast, we provide a rather simple, direct argument showing that no such dimension exists, within the standard model of set theory, without any additional assumptions.

\section{Query complexity of bandit learning}

In this work, we study query complexity of bandit learning. Specifically, we focus on the following problem. Let $\A$ be an action set, $\F$ a set of reward functions $f : \A \to[0,1]$, and ${f^*\in\F}$ the target reward function. 
In each round $t=1,\dotsc,T$, the learner queries an action $a_t\in\A$ and receives reward $r_t\in[0,1]$ with  $\E[r_t | a_t]=f^*(a_t)$. The goal is \textit{best-arm identification}: for a given $\epsilon\in[0,1]$, using as few queries as possible, identify an
$\epsilon$-optimal action. We consider both the \emph{noise-free} setting, where $r_t=f^*(a_t)$, and the \emph{noisy} setting,
where in each round $t=1,...,T$, the learner observes $r_t=f^*(a_t) + \xi$ for some  zero-mean random variable $\xi$.
Throughout the paper, unless stated otherwise, we will assume a Gaussian noise model, i.e., $\xi \sim \mathcal{N}(0,\sigma^2)$.

We say that a class of reward functions $\F \subseteq [0,1]^\A$ is \emph{bandit-learnable}  if there is a (possibly-randomized) algorithm $\Alg$ and a function $m: (0,1)^2 \rightarrow \mathbb{N}$ such that for any $f \in \F$, when given any $\epsilon,\delta >0$ and
    after having made at most $m(\epsilon, \delta)$ queries $a_t$ to $f$ and observed $r_t$ (under the appropriate noise model), algorithm $\Alg$ outputs $\hat{a}$ such that with probability at least $1-\delta$,  
    $$
    f(\hat{a}) \ge \sup_{a \in \A}f(a) -\epsilon.
    $$
    The function $m(\cdot, \cdot)$ is the \text{query complexity} of $\Alg$. We often denote $m_{\Alg}^\sigma(\cdot, \cdot)$ when considering noisy feedback, for the appropriate choice of $\sigma$. We then define the query complexity of a given class $\F$, for any fixed choice of parameters, as follows.
 \begin{definition}
Given $\epsilon, \delta \in [0,1]$, the 
  $(\epsilon,\delta)$-query complexity for class  $\F \subseteq [0,1]^\A$ under a Gaussian noise model with $\xi \sim \mathcal{N}(0,\sigma^2)$, denoted $\depth_{\epsilon,\delta}^\sigma(\F)$, is the minimum over all $m_{\Alg}^\sigma(\epsilon,\delta)$, where $m_{\Alg}$ is the query complexity of a bandit learning algorithm $\Alg$ for the class $\F$.
\end{definition}

\section{No combinatorial dimension can characterize bandit learnability}
\label{app:finite-char-property}

A fundamental result of statistical learning theory is the characterization of PAC learnability in terms of the VC dimension of a class. Similar combinatorial characterizations exist for diverse variants of statistical learning~\citep{Vapnik89, NatarajanT88, bartlett1994fat, ben1992characterizations, brukhim2022characterization} as well as online learning~\citep{littlestone1988learning, ben2009agnostic, rakhlin2015online, daniely2015multiclass}.

All standard notions of dimension in the aforementioned learning settings
can be abstracted as a function $\mathfrak{D}$ that maps a class of functions $\F \subseteq \Y^\X$
to $\mathbb{N} \cup \{ \infty\}$, while satisfying the following requirements: (1) \textit{learnability characterization}: a class $\F$ is learnable if and only if $\mathfrak{D}(\F) < \infty$, and (2)  \textit{finite character}: for every integer $d$ and $\F$, the statement ``$\mathfrak{D}(\F) \ge d$'' can be demonstrated by a finite set of domain points and a finite collection of members of $\F$. 
 We will next give a more formal definition of the finite character property. First, we define the notion of a shattered set. In the definition and throughout this section,
 we write $\F|_X$ to denote the set of all functions in $\F$ restricted to points in $X$.
 
\begin{definition}[Shattered sets]
 For every $d\in\mathbb{N}$ let $V_d: \X^d \times 
 2^{\Y^d}  \mapsto \{\text{\small YES, NO}\}$ be a shattering function. 
 A set  $X \in \X^d$ is shattered by hypothesis class $\F$, with respect to $V_d$ if and only if
 $\F|_X$ is of finite cardinality and $V_d(X, \F|_X) = \text{\small YES}$.
\end{definition}

{\begin{definition}[Finite character property]
\label{def:complexity-measure}
We say that a dimension $\mathfrak{D}$ satisfies the finite character property if 
for every $d\in\mathbb{N}$  there exists a shattering function $V_d$  such that $\mathfrak{D}(\F)\geq d$ if and only if there exists a shattered set of size at least $d$. 
\end{definition}}
This property was first defined by 
\cite{ben2019learnability}, who gave a {formal} definition of the notion of ``combinatorial dimension'' or ``complexity measure'', satisfied by all previously proposed dimensions in statistical learning theory. The intuition is that a finite character
property can be checked by probing finitely many elements of $\X$
and $\F$. For example, the classic VC dimension \citep{vapnik:74,Vapnik89} satisfies 
the finite character property since the statement  ``$\mathrm{VC}(\F)\geq d\mkern1mu$"  can be verified with a finite set of points $X = \{x_1\dots, x_d\} \subseteq \X$ and a finite set of classifiers $h_{1},\dots, h_{2^d}\in \F|_X$ that shatter $X$.

In a similar manner to the statistical setting, a dimension capturing \textit{bandit} learnability can be abstracted as a function $\mathfrak{D}$ that maps a class of {reward functions} $\F \subseteq \Y^\X$ to $\mathbb{N} \cup \{ \infty\}$.  We say the dimension $\mathfrak{D}$ satisfies the \emph{finite character} property if Definition \ref{def:complexity-measure} holds. We say the dimension \emph{characterizes bandit learnability} if for every integer $d$ and $\epsilon,\delta > 0$, there exists integers $m, M$ so that for every $\F$ the following holds: (1) if $\mathfrak{D}(\F) \ge d$,  then $\depth_{\epsilon,\delta}^0$ is at least $m$, and (2) if $\mathfrak{D}(\F) < d$, then $\depth_{\epsilon,\delta}^0$ is at most $M$.   The integers $m, M$ tend to $\infty$ as $d$ tends to $\infty$.

Perhaps the most well-known example of a combinatorial dimension in the context of bandit learning is the eluder dimension \citep{russo2013eluder}. Over more than a decade, it has been a central technique in the context of bandits as well as reinforcement learning (RL) \citep{li2022understanding, wang2020reinforcement, jin2021bellman}. It can be easily verified that the eluder dimension does satisfy the finite character property. However, it is also known there are arbitrarily large gaps between bounds obtained via the eluder dimension and related combinatorial measures \citep{brukhim2023unified}.

The following theorem shows that there is no non-trivial dimension that satisfies the finite-character property and
also characterizes bandit learnability. Our result holds regardless of the assumed cardinality of the continuum, and within standard ZFC set theory. Our findings complement the celebrated result from 
 \citep{hanneke2023bandit}, 
 that demonstrates a particular reward function class for which bandit learnability (or EMX learnability; \citealp{ben2019learnability}) depends on the cardinality of the continuum and is therefore
independent of the standard set theory ZFC axioms. Our result implies that even when we restrict our attention to classes for which bandit learnability is provable within ZFC, there cannot exist a dimension with the finite-character property characterizing bandit learnability.

\begin{theorem}[No finite-character dimension for bandits]\label{thm:fin_char}
Let $\X,\Y$ be arbitrary (possibly infinite) sets, of size $|\Y| \ge |\X| \ge d+1$  for some integer $d > 2$. Let $\mathfrak{D}$ be a  dimension for bandit classes in $\Y^\X$ that satisfies
the finite character property, and such that  
$\exists \F$ with $\mathfrak{D}(\F) \ge d$. 
Then, for any $\epsilon, \delta \ge 0$,   there exist  $\F \subseteq \Y^\X$ for which $\mathfrak{D}(\F) \ge d$, but the query complexity of bandit-learning $\F$ is bounded by $\depth_{\epsilon,\delta}^0(\F) \le 2$. In particular,  $\mathfrak{D}$ 
does not 
characterize bandit 
learnability.
\end{theorem}

\begin{proof} Consider a class $\F \subseteq \Y^\X$  such that $\mathfrak{D}(\F) \ge d > 2$. By the finite-character assumption, and since $\mathfrak{D}(\F)\geq d$, there exists a shattering function $V_d$, a set  $X = \{x_1,...,x_d\}$ and a set of vectors $F = \{v_1,...,v_n\} \in \F|_{X}$ for some integer $n$, that is, $F \in (\Y^d)^n$, such that $V_d(X, F) = \text{YES}$.

Since $|\X| > d$, there must exist a point  $x_0 \in \X$ such that $x_0 \notin X$.
We define a new class $\F'$ over the domain $\X$ as follows. For any $f \in \F$,  we define $f' \in \F'$ such that:
$$
f'(x)=
\begin{cases}
f(x) & \text{if } x \neq x_0,\\
\argmax_{x \in \X} f(x) & \text{if } x = x_0.
\end{cases}
$$

Then, the new class $\F' \subseteq \Y^\X$ consists of all functions $f'$ of the above form. We have that $|\F'|\le |\F|$. We now want to show the following two properties hold: (1) the query complexity of $\F'$ is at most $2$, and (2) $\mathfrak{D}(\F') \ge d$. It suffices to show (1) and (2) to complete the proof.

First, to show (1), notice that any algorithm can first query $x_0$ and obtain the value  $x_1 := \argmax_{x \in \X} f(x)$. Then, querying $x_1$ either immediately attains the optimal value of $f'$, or it may hold that $f(x_1) \le x_1$ in which case $x_0$ is the optimal value, since $f'(x_0) = x_1$. Thus, at most $2$ queries are needed to determine the optimal value of any $f' \in \F'$ up to any $\epsilon \ge 0$. 
 
Next, to show (2), simply observe that the set $F$ above is contained in $\F'|_X$. Then, since $V_d(X, F) = V_d(X, \F'|_X) = \text{YES}$  we get that $X$ is also shattered by hypothesis class $\F'$ with respect to $V_d$ and so  $\mathfrak{D}(\F') > d$. 
\end{proof}

\begin{remark}[``Reverse" finite character property]
    The property in \Cref{def:complexity-measure} requires that a \emph{lower} bound on the dimension be demonstrated by finitely many domain points $X$ and members of $\F|_{X}$. Indeed, as observed by \cite{ben2019learnability}, all standard notions of dimensions in statistical and online learning satisfy this property. One may also consider an alternative property which requires that an \emph{upper} bound on the dimension be demonstrated by finitely many domain points $X$ and members of $\F|_{X}$. However, 
    one can easily show that there cannot exist a dimension satisfying both this property and characterizing bandit learnability, for any infinite class.
\end{remark}

\section{Hardness of bandit learning}\label{sec:hard}

In this section, we study the computational efficiency of bandit learning in comparison to standard (albeit possibly computationally hard) algorithmic operations often considered in learning theory. A fundamental example is empirical risk minimization (ERM),
which can be used to find a hypothesis consistent with the observed data (as is sufficient, for instance, for PAC learnability).
In interactive learning settings, estimation algorithms are often used both to select consistent hypotheses and to make predictions (see, e.g., \citealp{foster2023tight, brukhim2023unified}). Given these, one might naturally expect that if a function class supports efficient algorithms for such tasks, it should also be efficiently learnable in the bandit setting.

Quite surprisingly, we prove that this intuition fails. We construct a class of reward functions where the optimal action can be identified with just two queries, yet no polynomial-time algorithm can achieve this, unless $\mathrm{RP}=\mathrm{NP}$. Furthermore, we show that this class does admit efficient algorithms for standard learning tasks, highlighting that in this case the computational hardness arises solely from the nature of the bandit-learning task.

\paragraph{Commonly used algorithmic procedures} Below we give $3$ definitions of the relevant algorithmic procedures we will consider in the main theorem presented in this section, \Cref{thm:hardness}. Specifically, we formally define a  consistency (ERM) algorithm, an online estimation algorithm, and a maximization algorithm, as follows.

\begin{definition}[Consistency (ERM) algorithm]
    An algorithm $\Alg$ is a \emph{consistency (ERM) algorithm} for a class $\F \subseteq [0,1]^\A$ if 
    for every $f \in \F$ and for every set $S = \{(a_1, f(a_1)),\dotsc,(a_m, f(a_m))\}$, where each $a_i \in \A$, when given $S$ as input, $\Alg$ returns $\hat{f} \in \F$ such that for all $i=1,\dotsc,m$ it holds that $f(x_i) = \hat{f}(x_i)$. 
\end{definition}

\begin{definition}[Online estimation algorithm]\label{assume:est:error}
An \emph{online estimation algorithm} for  $\F\subseteq [0,1]^\A$  is an algorithm that at each round $t=1,\dotsc,T$, when given a sequence of past observations $(a_{1},f(a_1)),\allowbreak\dotsc,(a_{t-1},f(a_{t-1})),$ 
for some $f\in \F$, it returns
an estimator $\hat{f}_{t}\in \mathcal{F}$.
The algorithm  has \emph{decaying estimation error} if there exists $\textbf{EST}(T) \ge 0$ growing sublinearly in~$T$, that is, $\textbf{EST}(T)=o(T)$, such that  for any sequence $a_1,\dotsc,a_T \in \A$,
we have  
\begin{equation}
\label{equation::least_squares_main_guarantee}
\sum_{t=1}^{T} \Big({ \hat{f}_{t}(a_t) - f(a_t)}\Big)^2 \leq \textbf{EST}(T).
\end{equation}
\end{definition}

\begin{definition}[Maximizing algorithm]
    An algorithm $\Alg$ for a class $\F \subseteq [0,1]^\A$ is a \emph{maximizing algorithm} if 
    for every $f \in \F$ and every $\epsilon > 0$, 
    it returns $\hat{a} \in \A$ such that
    $$
    f(\hat{a}) \ge \sup_{a \in \A}f(a) -\epsilon.
    $$
\end{definition}
A maximizing algorithm for a class $\F$ over a finite action set $\A$ 
is said to be 
{efficient}
if each function $f \in \F$ has a concise representation using $O(\text{poly-log}(|A|))$ bits, and the algorithm has running time that is polynomial in the size of the input, i.e., poly$(\log(|A|),1/\epsilon)$. \\

 \paragraph{Hardness of bandit learning} Recall that the complexity class RP (randomized polynomial time; \citealp{gill1974computational, valiant1985np}) is the class of decision problems solvable in polynomial time by a probabilistic Turing machine such that: if the answer is ``yes", at least $1/2$ of computation paths accept; if the answer is ``no", all computation paths reject. 
 
The following theorem demonstrates a reduction from the NP-complete problem of Boolean satisfiability to bandit learning, using a construction of a function class which at the same time allows efficient algorithms for standard learning algorithms. This establishes hardness of bandit learning, under the   assumption that $\mathrm{RP} \neq \mathrm{NP}$. 
 
\begin{theorem}[Hardness of bandit learning]\label{thm:hardness}
    For every $n \in \mathbb{N}$, there exists a finite function class $\F_n  \subseteq [0,1]^{\A_n}$ over action set $\A_n$ of size $2^{n+1}+1$, such that for every $\epsilon, \delta \ge 0$,
    $$
    \depth_{\epsilon, \delta}^0(\F_n) \le 2, 
    $$
    and such that the following holds. If there exists a bandit learning algorithm  for every $\F_n$ \   with running time that is polynomial in $n$, then $\mathrm{RP}=\mathrm{NP}$. \\ \\
    Moreover, each class $\F_n$ admits efficient deterministic algorithms as follows:
    \begin{itemize}[leftmargin=.5cm]
        \item The class $\F_n$ admits a   consistency (ERM) algorithm, of runtime $O(n^2)$.
        \item The class $\F_n$ admits an   online estimation algorithm, of runtime $O(n^2)$ and $\textbf{EST}(T) = O(1)$. 
        \item The class $\F_n$ admits a   
        maximizing algorithm, of runtime $\tilde{O}(n^2)$,  for every $\epsilon \ge 0$.
    \end{itemize}
\end{theorem}

\begin{remark}   We remark that although \Cref{thm:hardness} is stated in the noise-free setting, a similar result can also be proved in the noisy setting. First, it can be shown that
    Gaussian noise model with sufficiently low variance $\sigma \approx 1/2^n$ is not qualitatively different from the noise-free case.\footnote{The proof follows similarly to that of \Cref{thm:noisy_depth_is_det_depth_for_low_noise}. See \Cref{sec:noisy_vs_det} for further discussion.} In particular, for such small values of $\sigma$, we obtain $\depth_{\epsilon, \delta}^\sigma(\F_n) \le 2$ as well as all other statements from \Cref{thm:hardness}, where the guarantees for efficient algorithms now hold with high probability. More generally, our construction exhibits a trade-off between the optimal query complexity and the variance of the noise model, such that a large-variance noise model can be incorporated while increasing the optimal query complexity. In particular, \Cref{thm:hardness} could be extended to the noisy setting under Gaussian noise
with large, \emph{constant} variance (e.g., $\sigma = 1$), but query complexity of order  $\depth_{\epsilon,\delta}^\sigma = \tilde{O}(n^2)$, for every $\epsilon,\delta$. Thus, although the optimal $\mathrm{QC}$ is polynomial in $n$, a similar construction as shown below demonstrates that there is no bandit learning algorithm that runs in polynomial time, unless $\mathrm{RP} = \mathrm{NP}$. See \Cref{sec:noisy_vs_det} for related results and discussion. 
\end{remark}

\begin{proof}[Proof of \Cref{thm:hardness}]
Throughout the proof, we fix $n \in \mathbb{N}$ and simply denote $\A_n, \F_n$ by $\A,\F$, for brevity.  We start by defining $\A$ as follows: $\A = \{\star\} \cup \A^{(2)} \cup \A^{(3)}$, such that   $\A^{(2)} = \{0,1\}^n$, and $\A^{(3)} = [2^n]$. 
{We will construct a class $\F \subseteq [0,1]^\A$ which 
is best thought of as represented by a query tree of the following structure:  $\star$ corresponds to the root node, and actions in $\A^{(2)}$ and  $\A^{(3)}$ correspond to nodes of the second and third layer of the tree, respectively.} Before defining the class~$\F$, we consider the following set:
 $$
 \Phi = \{ \text{all 3CNF formulas }\phi \text{ on } n \text{ variables and at most }n^2 \text{ clauses} \}.
 $$
For every $\phi \in \Phi$ that is satisfiable, we denote by $a^*_\phi$ the satisfying assignment for $\phi$ that is minimal according to the natural ordering on $\A^{(2)}$. Define $\F \subseteq [0,1]^\A$ as follows:
$\F = \F^{sat} \cup \F^{all}$ \  where,
$$
\F^{all} = \{ f_{\phi}:  \forall \phi \in \Phi \}, \qquad \text{ and }\qquad \F^{sat} = \{ f_{\phi, c}:  \forall \ \phi \in \Phi \text{ s.t. }\phi \text{ 
 is satisfiable}, \ c \in \A^{(3)}  \},
$$
and where the functions of the form $f_{\phi, c}$ and $f_{\phi}$ are defined as follows: 
 \begin{equation*}
  f_{\phi, c}(a) = \begin{cases}
            \text{encode}(\phi)  & \text{if } a = \star \\
            \frac{1}{2^{n+1}}\cdot
            c & \text{if } a = a^*_\phi \in  \A^{(2)} \\
            1   & \text{if } a = c \in  \A^{(3)}   \\
            0 & \text{otherwise, }\
    \end{cases}  
\hspace{1cm} \text{ and } \hspace{1cm}
    f_{\phi}(a) = \begin{cases}
            \text{encode}(\phi)  & \text{if } a = \star \\
            0 & \text{otherwise, }
    \end{cases} 
\end{equation*}
where $\text{encode}(\phi)$ encodes the formula by some value in $[\frac{1}{4}, \frac{1}{2}]$. For example, $\text{encode}(\cdot)$ can be implemented as follows. Each {literal} can first be encoded using \( \log(n) + 1 \) bits (the variable index plus 1 bit for negation). The full formula requires \( O(n^2 \log(n)) \) bits, and this binary string could be embedded in $[0,1/5)$ by writing it after the decimal point. Lastly, this value could then be shifted by $1/4$ so that it lies in the desired range $[\frac{1}{4}, \frac{1}{2}]$. This encoding can be easily decoded by any learner if there is no noise added to the encoded value $f_{\phi, c}(\star)$.

\paragraph{Query complexity $2$:}
Let us argue that the query complexity of this class $\F$ is indeed at most $2$. 
Specifically, we will describe a deterministic  algorithm $\Alg$ for $\F$ such that for any $f \in \F$ it requires at most $2$ queries to recover the optimal action.
First, $\Alg$ queries $a = \star$ and observes the encoding  $\text{encode}(\phi)$, which allows it to recover the formula $\phi$. 
Then, by brute force search over all assignments $a \in \A^{(2)}$, it can obtain   $a^*_\phi$, if there is one, and if there is none, the optimal action is simply $\star$. If $\phi$ is satisfiable, $\Alg$ will query $a = a^*_\phi$ and if the value is $0$, then the optimal action is again $\star$. Otherwise, $\Alg$ observes $\frac{1}{2^{n+2}}\cdot c$. Thus, it has recovered the optimal action $c$ for the function $f$, with only $2$ queries. The third query of $a = c$ will then yield the optimal value.  

\paragraph{Hardness:} We prove hardness for any bandit learning algorithm $B$ for $\F$. Fix $\epsilon=1/10$, and note that by the construction of the class, any algorithm that finds an $\epsilon$-optimal action, has actually found an optimal one. Let $a_B$ denote the final query submitted by any algorithm $B$.  Assume towards contradiction that there exists an algorithm $B$ such that for every $f \in \F$, by using only poly$(n)$ runtime, it outputs an action $a_B$ such that: 
 \begin{equation}\label{eq:B_good}
    \Pr_{B} \left[f(a_B) = \max_{a \in \A}f(a)  \right] \ge 3/4. 
    \end{equation}
  
 We will prove this solves the SAT decision problem in poly$(n)$ time and in a probabilistic manner, demonstrating that this NP-complete problem is in $\mathrm{RP}$, in contradiction to the assumption that $\mathrm{RP} \neq \mathrm{NP}$. Specifically, we will describe how, given access to $B$, one can construct an algorithm so that for every $\phi \in \Phi$, if it is satisfiable the algorithm accepts (declares "yes") with probability at least $1/2$, and if not - the algorithm always rejects (declares ``no").

Given any formula $\phi$, we simulate running the algorithm $B$ by responding exactly as if $f_{\phi} \in \F^{all}$ would respond. Specifically, for each query made by $B$ we respond as follows: if the query is $\star$, we respond with $\text{encode}(\star)$, and for any other we respond with $0$, 
until either $B$ queries for $a \in \A^{(2)}$ which is a satisfying assignment for $\phi$ (which we can easily verify efficiently), in which case we halt the simulation, or until $B$ terminates and returns its final query $a_B$.

We will now show that our simulation can solve the SAT decision problem with a one-sided error with constant probability, as detailed next. First, assume $\phi$ is not satisfiable. Then, $B$ can never query for   $a \in \A^{(2)}$ which is a satisfying assignment for $\phi$, and so we would run the simulation until $B$ terminates, in poly$(n)$ time, after which we declare that $\phi$ is not satisfiable. This always occurs, thus whenever $\phi$ is not satisfiable then with probability $1$ we reject.  \\

\noindent 
Now, assume $\phi$ is satisfiable. We have the following lemma, whose proof is deferred to the appendix.
\begin{lemma}\label{lemma:B_bad_sat}
    Let $\F$ as constructed above, and let $B$ be any bandit learner for $\F$ as above (i.e., for every $f \in \F$ its output satisfies \Cref{eq:B_good}). Fix any $\phi \in \Phi$ that is satisfiable. 
    Then, there exists $c \in \A^{(3)}$ such that if $B$ is being run with $f_{\phi,c}$ and $a_1,\dotsc,a_m$ denotes its query sequence during that run, it holds that: 
$$
  \Pr_{B} \left[\exists i\in [m], \ \ \ \ a_i \text{ is a satisfying assignment for }\phi \ \ \land \ \  \forall j < i, \ a_j \neq c   \right] \ge  \frac{3}{4} - \frac{2m}{2^n}.
$$
\end{lemma} 
Then, by Lemma \ref{lemma:B_bad_sat} we have that there exists some $c \in \A^{(3)}$ such that when $B$ interacting with $f_{\phi, c}$,
and $a_1,\dotsc,a_m$ denotes its query sequence during that run, 
then :
    $$
  \Pr_{B} \left[\exists i\in [m], \ \ \ \ a_i \text{ is a satisfying assignment for }\phi \ \ \land \ \  \forall j \le i, \ a_j \neq c   \right] \ge  \frac{1}{2},
    $$
since $m$ is some polynomial in $n$, then for all sufficiently large $n$ we have $\frac{2m}{2^n} \le 1/10$. %
{Importantly, we do not need to know what this $c$ is during simulation. The reason is that, by the above, we have that with probability at least $1/2$  we will be able to simulate a response sequence by $f_{\phi, c}$ since it will be identical to the response sequence by $f_{\phi}$, until we observe a satisfying assignment, in which case we halt. Thus, with probability at least $1/2$ we 
will observe a satisfying assignment, and declare "yes". Notice that with probability $< 1/2$ our simulation will not be consistent with $f_{\phi, c}$ but that is of no concern to us, as we may reject in this case. It is, however, crucial that $B$ runs in poly$(n)$ time even when interacting with $f_\phi$ rather than with $f_{\phi,c}$, which indeed holds as $f_\phi \in \F$.   }

The proof of the theorem is then concluded by
proving the existence of efficient algorithms for the class $\F$, which holds by Lemma \ref{lemma:sat_oracles}, given in the appendix. 
\end{proof}

\section{Noise-free  vs. noisy setting query complexity}\label{sec:noisy_vs_det}

The first question we address is whether there exists any provable relationship between the noise-free query complexity $\depth^0_{\epsilon, \delta}(\F)$ and $\depth^\sigma_{\epsilon, \delta}(\F)$. First we show there exist function classes for which their noise-free  query complexity is constant but such that $(\epsilon,\delta)$-complexity is unbounded.

\begin{restatable}{propositionx}{propositiondeterministicnoisydepthgap}\label{proposition::deterministic_vs_noisy_depth_gap}
    Given $\epsilon \in [0,1/2)$, there exists a function class $\F$ such that $\depth^{0}_{\epsilon, \delta'}(\F) = 1$ for all $\delta' \in [0,1)$ but  $\depth^{\sigma}_{\epsilon, \delta}(\F) = \infty$ for all $ \delta \in [0, 1/2)$ and all $\sigma > 0$. 
\end{restatable}

The function class $\F$ used to prove Proposition~\ref{proposition::deterministic_vs_noisy_depth_gap} is based on a ``informative action'' construction where the action space equals $\{ 0\}\cup \mathbb{N}$. The optimal action of any function in $\F$ is indexed by $n \in \mathbb{N}$. Action $0$ is ``informative'' because its mean reward reveals the identity of the optimal action so that $\depth_{\epsilon, \delta'}(\F)=1$. Nonetheless in the noisy setting when $n$ goes to infinity, estimating the reward of action $0$ or finding the optimal action via enumeration of $\mathbb{N}$ requires a number of queries growing with $n$. The formal proof is given in Appendix~\ref{appendix:proof:deterministic_vs_noisy_depth_gap}.

Upper and lower bound bounds for $\depth_{\epsilon, \delta}^1(\F)$ were derived by~\cite{hanneke2024complete} for the high-noise regime (i.e., when $\sigma$ is of order $1$ for functions with values in $[0,1]$) based on the generalized maximin volume $\gamma_{\F, \epsilon}$ of $\F$ (see definition below). 

\begin{definition}[Generalized maximin volume; \citealp{hanneke2024complete}]
\label{def:gamma_dim}
\emph{Generalized maximin volume} of a function class $\mathcal{F}$ is defined as
\begin{equation}
\label{eq:gamma_dim}
    \begin{aligned}
\gamma_{\F,\epsilon}=\adjustlimits\sup_{p \in \Delta(\A)} \inf_{\;f \in \mathcal{F}\;} \mathbb{P}_{a \sim p}\left(\sup_{a^*}f(a^*)-f(a)\leq \epsilon\right),
    \end{aligned}
\end{equation}
where $\Delta(\A)$ is the set of all distributions on $\A$.
\end{definition}
Theorem~1 of~\cite{hanneke2024complete} presents an elegant and insightful result, establishing that $\depth_{\epsilon,\delta}^1(\F)$ can be lower bounded by $\Omega\left(\log(1/\gamma_{\F, \epsilon})\right)$ and upper bounded (up to constant and logarithmic factors) by $1/(\gamma_{\F, \epsilon/2} \cdot \epsilon^2)$. In this work we explore the low-noise regime where these results break down. In Theorem~\ref{theorem::noisy_depth_lower_bounds_det_depth} we show among other things that for any $K \in \mathbb{N}$ and $\epsilon \in [0,1/2)$ there is a function class $\F$ such that $\gamma_{\F,\epsilon} = K$ but there exist values of $\sigma > 0$ where $\depth^\sigma_{\epsilon, 1/4}(\F) = 1 < \log(1/\gamma_{\F, \epsilon})$. This result shows $\depth_{\F, \epsilon}^\sigma(\F)$ behaves fundamentally differently in the low-noise and high-noise regimes, highlighting the need for better theories to understand this phase transition. %

\begin{restatable}{theorem}{theoremnoisydepthann}\label{theorem::noisy_depth_lower_bounds_det_depth}
 There exist universal constants $c, \bar{c} > 0$ such that for every integer $K \ge 2$ there exists a function class $\F \subseteq [0,1]^{\A}$  with action space $|\A | = K+1$ such that   for every $\epsilon \in [0,1/2)$ it holds that $\gamma_{\F, \epsilon} = 1/K$  and if $\sigma^2 \geq \frac{1}{\bar{c}K^{2/3}}$ then
\begin{equation*}
    \bar{c} K^{2/3} \sigma^2 \leq   \depth_{\epsilon, 1/4}^\sigma(\F) \leq c \log^{2/3}(K) K^{2/3}\sigma^2.
\end{equation*}
In particular, 
\begin{equation*}
   \bar{c}\log(1/\gamma_{\F,\epsilon}) \sigma^2 \leq    \depth^{\sigma}_{\epsilon, 1/4}(\F),
\end{equation*}
and if  $\sigma^2 \leq \frac{1}{c\log^{2/3}(K) K^{2/3}}$ then
$$
\depth_{\epsilon, 1/4}^\sigma(\F)  = \depth_{\epsilon, 0}^0(\F) = 1.
$$
\end{restatable}

Similar to Proposition~\ref{proposition::deterministic_vs_noisy_depth_gap}, the function class $\F$ in Theorem~\ref{theorem::noisy_depth_lower_bounds_det_depth} is constructed around an ``informative action" structure, where the action space is given by \(\{0\} \cup [K]\). The optimal action belongs to $[K]$, while the mean reward of action \(0\) (the informative action) reveals its identity. This construction differs from Proposition~\ref{proposition::deterministic_vs_noisy_depth_gap} in its encoding representation. Specifically, our design ensures that strategies leveraging the information encoded in the mean reward of action $0$ achieve greater efficiency compared to the $\mathcal{O}(\sigma^2 K)$ queries needed by a strategy that individually estimates the mean rewards of all actions in $[K]$. Interestingly, in the large-variance regime the optimal data collection strategy that achieves the lower bound rate works by querying action $0$ sufficiently to narrow down the optimal action choices to $\mathcal{O}(K^{2/3})$ actions. When the noise variance is sufficiently small, the mean reward encoded by action $0$ can be inferred from a noisy sample  with probability of error at most $1/4$, leading to a query complexity of just $1$. These results highlight the intricate balance between exploiting the information structure of the function class—encoded here by action \(0\)—and relying on brute-force exploration by following the policy dictated by the generalized maximin volume in Equation~\ref{eq:gamma_dim}. The formal proof is given in Appendix~\ref{appendix::gap_proof}.

Building on these results we introduce the $(\epsilon,\delta)$-gap of a function class $\F$, denoted $\mathrm{Gap}_{\epsilon,\delta}(\F)$, which we then use to derive sufficient conditions on $\sigma$ to guarantee $\depth_{\epsilon,\delta}^\sigma(\F) \lesssim \depth_{\epsilon,\delta}^0(\F)$.

\begin{definition}[\bf Informal: Gap of $\F$] 
Let $\F \subseteq [0,1]^\A$ be a finite function class and action space. Given $\epsilon, \delta \in [0,1]$, we define $\mathrm{Gap}_{\epsilon,\delta}(\F)$ as the smallest difference between achievable function values for an action that an $(\epsilon, \delta)$-optimal algorithm might play, given any positive probability history.\looseness=-1
\end{definition}
Definition~\ref{def:gap_F} in Appendix~\ref{appendix:proof:deterministic_vs_noisy_depth_gap} formalizes the description above. Our next result establishes that when $\sigma$ is small, the $(\epsilon,\delta)$-query complexity of $\F$ is not too different from that of the noise-free $(\epsilon,\delta')$-query complexity of $\F$ provided that $\delta' < \delta$.

\begin{restatable}{theorem}{theoremgapnoisynoiseless}\label{thm:noisy_depth_is_det_depth_for_low_noise}
    Let $\delta, \delta' \in (0,1)$ such that $\delta > \delta' \geq 0$.
   For any finite class $\F \subseteq [0,1]^\A$ over a finite action space the noisy query complexity with zero-mean Gaussian noise with variance $\sigma^2$ such that $\sigma^2 < \frac{\mathrm{Gap}^2_{\epsilon, \delta'}( \F) }{4\log(2\depth^0_{\epsilon, \delta'}(\F)/(\delta-\delta'))}$ satisfies:
   \begin{equation*}
       \depth_{\epsilon,\delta}^\sigma(\F) \leq \depth^0_{\epsilon, \delta'}(\F) 
   \end{equation*}
\end{restatable}

To prove theorem~\ref{thm:noisy_depth_is_det_depth_for_low_noise} we show that we can construct a noisy feedback algorithm $\Alg$ based on an $(\epsilon,\delta')$-optimal noise-free algorithm $\Alg'$ that is guaranteed to have an error probability of at most $\delta$. $\Alg$ uses nearest neighbors to transform noisy rewards into mean reward values. When the noise is small $\Alg$ recovers the correct mean rewards with an error probability at most $\delta-\delta'$. These rewards are fed into a copy of $\Alg'$ and the suggested action exploration policies are executed. The resulting algorithm achieves the same query complexity as $\Alg'$ with a slightly degraded error upper bound of $\delta$. The inequality $\delta' < \delta$ is required because for any level of non-zero gaussian noise ($\sigma > 0$) translating noisy rewards into their mean reward values will necessarily produce an irreducible error probability. The proof of Theorem~\ref{thm:noisy_depth_is_det_depth_for_low_noise} is given in Appendix~\ref{appendix:proof:deterministic_vs_noisy_depth_gap}.

\section{Separation between regret and query complexity}\label{sec:sep}

We study the separation between regret and query complexity, both in the noise-free and noisy settings. The regret of an algorithm $\Alg$ with action space $\A$, interacting for $T$ rounds by producing an action $a_t \in \A$ for $t =1,\dotsc,T$ and observing rewards generated by $f^*$ is defined as,
\begin{equation*}
    \mathrm{Regret}_{\Alg}(T) = \sum_{t=1}^T \max_{a\in\A} f^*(a) - f^*(a_t).
\end{equation*}
A typical objective in the bandit online learning and reinforcement learning literature is to design algorithms that satisfy a sublinear regret bound such that $\lim_{T\rightarrow \infty} \frac{\mathrm{Regret}(T)}{T} = 0$. In this section, we explore whether achieving low query complexity and low regret are compatible objectives. 
We show negative results in this regard in the noise-free (Appendix~\ref{app::noise_free_discussion}) and noisy settings (Section~\ref{section::regret_vs_QC_noisy}). In each of these scenarios, we show that it is impossible to construct algorithms that achieve optimal query complexity while also incurring sublinear regret. This holds because, in certain problems, any optimal algorithm for $\epsilon$-arm identification must allocate a significant number of queries to actions that, while highly informative, result in substantial regret.

\subsection{Regret vs. QC: noisy case}\label{section::regret_vs_QC_noisy}

In this section we explore the compatibility of the optimal query complexity and regret minimization in noisy feedback problems. Similar to our results in the noise-free setting in Theorem~\ref{theorem::main_theorem_result_noisy_separation} we show there are problems where the goal of finding an optimal action cannot be achieved without paying a regret scaling linearly with the query complexity; however, for the same function classes, there is an algorithm achieving regret scaling as the square root of the number of time-steps. 

\begin{restatable}{theorem}{theoremmainresultnoisyseparation}\label{theorem::main_theorem_result_noisy_separation}
    Let $d,T \in \mathbb{N} $. There exists a function class $\F$ over action space $\A$ with unit-variance Gaussian noise such that $ d \leq \depth^1_{0, 1/4}(\F) \leq 80d$ and any algorithm $\Alg$ such that $m^1_{\Alg}(0,1/4) \leq T$ satisfies 
\begin{equation*}
 \max_{f \in \F}   \mathbb{E}_\Alg[ \mathrm{Regret}(T,f) ] \geq \frac{d}{128}.
\end{equation*}
Moreover, there is an algorithm $\Alg'$ that satisfies $\max_{f \in \F} \mathbb{E}_{\Alg'}\left[ \mathrm{Regret}(T, f) \right] \leq 8\sqrt{2T\log(T)}$ for all $T \in \mathbb{N}$.  
\end{restatable}

Theorem~\ref{theorem::main_theorem_result_noisy_separation} suggests there exists a function class with query complexity $\mathcal{O}(d)$ such that when $T = \mathcal{O}(d^\alpha)$ for $\alpha <2$ then no algorithm $\Alg$ that is able to find an optimal action in $T$ queries can also satisfy an $\tilde{\mathcal{O}}(\sqrt{T})$ regret bound. Nonetheless, for the same function class, there are algorithms that achieve $\tilde{\mathcal{O}}(\sqrt{T})$ regret bounds. Theorem~\ref{theorem::main_theorem_result_noisy_separation} is closely related to Theorem~1 of~\citet{bubeck2011pure}. While Theorem~1 of~\citet{bubeck2011pure} rules out the existence of algorithms that achieve optimal regret and simple regret simultaneously, Theorem~\ref{theorem::main_theorem_result_noisy_separation} establishes that no algorithm can achieve both optimal query complexity and optimal regret.  Although related, the notions of simple regret and query complexity are different. 
The \emph{simple regret} for a fixed horizon $T$ is the expected gap between the algorithm’s output arm $a_T$ and the optimal arm $a^*$. In contrast, \emph{query complexity} can be thought of as the minimum horizon $T$ such that the optimal simple regret is at most $\epsilon$.

The function class $\F$ used to prove Theorem~\ref{theorem::main_theorem_result_noisy_separation} has an ``information lock'' structure. The action space is divided into two sets $\A_1$ and $\A_2$. The values of the mean rewards of actions in $\A_1$ can be used to infer the identity of the mean optimal action. Actions in $\A_1$ have large regret and their mean rewards are equal to $1/2 + \epsilon_1$ or $1/2-\epsilon_1$ while the mean rewards of actions in $\A_2$ are equal to $1$ or $1-\epsilon_2$ for parameters $\epsilon_1, \epsilon_2 \in [0,1]$ such that $\epsilon_1 \geq \epsilon_2$. %

To prove Theorem~\ref{theorem::main_theorem_result_noisy_separation} we first establish that $\depth_{0, 1/4}^1(\F) = \Theta(1/\epsilon_1^2) $. Second, we show that when $\epsilon_2 \approx \epsilon_1^2$, then any algorithm $\Alg$ such that $m_\Alg^1(0,1/4) \leq T$ must also incur regret satisfying $\max_{f \in \F} \mathbb{E}[  \mathrm{Regret}(T,f) ] \geq \Omega(1/\epsilon_1^2)$. The proof of Theorem~\ref{theorem::main_theorem_result_noisy_separation} follows by setting $\epsilon_1 \approx 1/\sqrt{d}$. Finally, since the problem in this class is an instance of multi-armed bandits, the UCB algorithm is guaranteed to collect sublinear regret.  The formal proof of Theorem~\ref{theorem::main_theorem_result_noisy_separation} can be found in Appendix~\ref{section::proof_theorem_separation_reg_noisy}. In Appendix~\ref{app::noise_free_discussion}, we establish analogous results for the noise-free setting. Notably, these findings do \emph{not} follow directly from Theorem~\ref{theorem::main_theorem_result_noisy_separation}. While Theorem~\ref{theorem::main_theorem_result_noisy_separation} is stated for \(\sigma = 1\), the query complexity in this construction approaches \(1\) as \(\sigma\) tends to zero, preventing a straightforward extension to the noise-free case.

\section{Conclusion}

In this work, we have presented new insights into the study of the learnability of structured bandit problems, shedding light on the interaction between their statistical and computational properties. Our main results highlight fundamental distinctions between classical learnability as studied in statistical learning theory and learnability in the bandit setting.   

We show that there cannot exist a combinatorial finite-character dimension that fully characterizes bandit learnability, a result that sets apart this setting from standard PAC learnability. We also prove there cannot exist query optimal algorithms that are computationally tractable even with access to standard algorithmic primitives such as empirical risk minimization and function maximization oracles. 

We also investigated the effects of observation noise on the query complexity of bandit problems. We show that there are function classes where a small amount of observation noise leaves query complexity unaffected, alongside classes where any amount of noise makes bandit learnability impossible. Finally, we demonstrate that there is a sharp distinction between algorithms adapted for query complexity and regret minimization: There are no algorithms that can simultaneously minimize these two objectives. 

We hope that by documenting these phenomena, we can help advance the community's understanding of bandit learnability and guide future algorithmic design and theoretical advancements in this area.

\acks{A.P. thanks Alessio Russo for helpful discussions.}

\newpage
\bibliography{bib}
\newpage
\appendix

\section{Missing proof of \Cref{sec:hard}}

\begin{lemma}\label{lemma:sat_oracles}
    Let $\F$ as constructed above. Then, the following holds:
    \begin{itemize}[leftmargin=.5cm]
        \item The class $\F_n$ admits a   
        maximizing algorithm, of runtime $\tilde{O}(n^2)$
        for every $\epsilon \ge 0$.
        \item The class $\F_n$ admits a   consistency (ERM) algorithm, of runtime $O(n^2)$.
        \item The class $\F_n$ admits an   online estimation algorithm, of runtime $O(n^2)$ and $\textbf{EST}(T) = O(1)$.
    \end{itemize}
\end{lemma}

\begin{proof}[Proof of Lemma \ref{lemma:sat_oracles}]
Consider the class $\F$ as constructed in the proof of \Cref{thm:hardness}.  

\paragraph{Consistency (ERM) algorithm:} The class $\F$ admits an efficient consistency algorithm, as follows.
Let $S = \{(a_1, r_1),...,(a_k, r_k)\}$ denote a sequence of $k \ge 0$ queries and their corresponding values, consistent with some $f \in \F$, such that $S$ is given as input to the consistency algorithm.  There are $3$ types of actions for which the feedback could be non-zero, and so there are $2^3=8$ total possibilities to consider. For the trivial case of $r_i=0$ for all $i \in [k]$, then any $f \in \F^{all}$ is consistent.  

Then, if the only action $a_i$ for which $r_i \neq 0$ is $a_i=\star$,  we can simply return $f_\phi$ for $\phi$ recovered from $r_i = \text{encode}(\star)$. If $S$ contains both $a_i=\star$, and $a_j \in \A^{(2)}$ such that $r_j \neq 0$, then the true function can be fully recovered from $S$ (since 
$r_j = c/ 2^{n+1}$),
and so the algorithm can return $f_{\phi,c}$. Similarly, if $S$ contains both $a_i=\star$, and $a_j \in \A^{(3)}$ such that $r_j \neq 0$, then the true function can again be fully recovered from $S$ (since 
$a_j=c$).

Next, if $\star$ was not queried, but $S$ indicates that some $a_{i}$ is the minimal satisfying assignment, then we can construct a formula $\phi$ so that $\phi(x)$ is true if and only if $x=a_i$, so that $a_{i}$ is $\phi$’s only satisfying assignment.  Constructing such a boolean formula is straightforward: for each clause $j=1,...n$ we take a disjunction over the literal $b_j$ such that the assignment $a_i(j)$ is true, repeated $3$ times. Then, we let $\phi$ be the conjunction of those $n$ clauses.

If the only non-zero observed value is $f(a)=1$ then we can take any $\phi \in \Phi$ and $c=a$, and return $f_{\phi,c}$. If both $f(a)=1$ and $f(a') \neq 0$ have been observed, then we can repeat the construction of $\phi$ as above, for which $a'$ the only satisfying assignment, and set $c = a$. Thus, we have described an efficient consistency (ERM) algorithm for the class $\F$.

\paragraph{Online estimation algorithm:} 
The class $\F$ admits an online estimation algorithm, as follows. First, by running a consistency procedure as described above, then for any sequence of past observations we get $\hat{f} \in \F$ consistent with it. Thus, in every round $t=1,...,T$, we obtain a function $\hat{f}_t \in \F$ consistent with all observations up to time $t-1$. Then, since 
for any ground truth function $f \in \F$ used to generate the data sequence there are at most $3$ types of actions for which the feedback could be non-zero, and since all values are bounded in $[0,1]$,  the predicting error of the algorithm is at most $3$. Thus, for any integer $T$, $\textbf{EST}(T) \le 3$.

\paragraph{Maximization algorithm:} 
Lastly, notice that $f_{\phi, c}$ has a concise representation using at most $O(n^2\cdot \log(n))$ bits, used to represent both the formula $\phi$ and the value $c$. Thus, the class $\F$ admits an efficient maximizing algorithm - when given its concise representation once can easily recover the optimal action $c$ by reading it from the function description. 

\end{proof}

\subsection{Proof of \Cref{lemma:B_bad_sat}}

Since $\phi$ is satisfiable, we know that the optimal query for any $f_{\phi,c}$ is $c$. We also denote by $a_B^{c}$ the final output of $B$ after having interacted with  $f_{\phi,c}$, for some fixed $c$. Denote the event that "$\exists i \in [m], \text{ s.t. } a_i \text{ is a satisfying assignment for }\phi$" by $B$-good and its converse by $B$-bad. Then, for any fixed $c$, if $B$ is being run with some $f_{\phi,c}$ it holds that,
    \begin{align}\label{eq:lemma_ineq}
             \Pr_{B} \left[f(a_B) = \max_{a \in \A}f(a)  \right]  &=   \Pr_{B} \left[a_B^c = c\right] \\ \nonumber
              &=   \Pr_{B} \left[a_B^c = c \ \land \ B\text{-good} \right] + \Pr_{B} \left[a_B^c = c \ \land \ B\text{-bad} \right]\\ \nonumber
            &\le   \Pr_{B} \left[ B\text{-good} \right] + \Pr_{B} \left[\exists i\in [m], \ a_i = c \ \land \ B\text{-bad} \right].
    \end{align}     
    Next, we consider the uniform distribution $U$ over all $c \in \A^{(3)}$. We use $R \sim \R$ to denote the internal randomization used by the algorithm $B$. We may then view $B$'s computation of 
    its final output $a_B$ as a   fixed and deterministic function of $R$, and of the observed values from $f_{\phi,c}$. Then,
        \begin{align*}
            \E_{c \sim U} & \Big[ \Pr_{B} \left[\exists i\in [m], \ a_i = c \ \land \ B\text{-bad} \right]\Big] \\
            &=     \E_{R \sim \R}
            \left[\Pr_{c \sim U} 
             \Big[ \exists i\in [m], \ a_i = c \ \land \ B\text{-bad}  
             \ | \  R\Big] \right] \\
              &=   
            \E_{R \sim \R}
            \left[\Pr_{c \sim U} 
             \Big[\exists i\in [m], \ a_i = c \ | \  R, B\text{-bad}\Big] \cdot \Pr_{c \sim U} [B\text{-bad}|R] \right] \\
            &\le \frac{m}{2^n},
    \end{align*}

where the equalities follow from linearity of expectation and law of total expectation, and the inequality follows from the following fact; When $R$ is fixed  (so $B$ is viewed as a deterministic function) and conditioned on the event that $B$ is $B$-bad for $c$,  then notice that during the interaction of $B$ with $f_{\phi,c}$ it only observed $0$ for all of its queries (except $\star$). 
{Thus, $a_B$ is only a deterministic function of $\phi$ and a sequence of queries $(a_i, 0)$, unless any query happened to "hit" $c$. 
Since $c$ is chosen at random independently of $\phi$, the probability of any deterministic choice out of all elements in $\A^{(3)}$ will happen to be $c$ is $1/2^n$, and for $m$ such choices to probability is at most $m/2^n$.} Then, this implies that there exists a \textit{particular} $c \in \A^{(3)}$ for which, 
\begin{equation}\label{eq:lemma_some_eq_1}
     \Pr_{B} \left[\exists i\in [m], \ a_i = c \ \land \ B\text{-bad} \right] \le  \frac{m}{2^n}.
\end{equation} 
Next, combining this inequality with \Cref{eq:B_good} and \Cref{eq:lemma_ineq} we get, 
 \begin{equation}\label{eq:lemma_some_eq_2}
  \Pr_{B} \left[ B\text{-good} \right] \ge   \frac{3}{4} - \frac{m}{2^n}.
\end{equation}

Observe that the same reasoning used for the proof of \Cref{eq:lemma_some_eq_1} also holds for any $m' \le m$ (where $B\text{-bad}$ is modified to only consider queries up to $m'$ as well). Therefore, for all $m' \le m$, 
\begin{equation}\label{eq:lemma_some_eq_3}
      \Pr_{B} \left[\exists i\in [m'], \ a_i = c \ \land \  \forall i \in [m'], \ a_i \text{ is not a satisfying assignment for }\phi \right] \le   \frac{m'}{2^n}.
\end{equation} 
Then, combining \Cref{eq:lemma_some_eq_2} and \Cref{eq:lemma_some_eq_3} we get: 
{
\begin{align*}
  \frac{3}{4} - \frac{m}{2^n}  &\le \Pr[\text{B-good}] = \Pr[\exists  i\in [m], a_i \text{ satisfying }\phi] \\
  &= \Pr[\exists  i\in [m], a_i \text{ satisfying }\phi    \land  \ \forall   a_i \text{ satisfying }\phi,  \exists\  j < i, a_j = c]  & \tag{\textcolor{gray}{\small{$c$ queried before $a_i$}}}\\
   & \qquad  + \Pr\left[\exists i\in [m], \ \ \ \ a_i \text{ satisfying }\phi \ \ \land \ \  \forall j < i, \ a_j \neq c  \right]  & \tag{\textcolor{gray}{\small{$c$ not queried before $a_i$}}}\\ 
    &\le \frac{m}{2^n} +  \Pr\left[\exists i\in [m], \ \ \ \ a_i \text{ satisfying }\phi \ \ \land \ \  \forall j < i, \ a_j \neq c  \right]. 
\end{align*}}
Re-arranging yields the desired claim.

\section{Missing proofs of \Cref{sec:noisy_vs_det}}

\subsection{Proof of Proposition~\ref{proposition::deterministic_vs_noisy_depth_gap}} \label{appendix:proof:deterministic_vs_noisy_depth_gap}

\propositiondeterministicnoisydepthgap*

Let $\A = \{0\} \cup \mathbb{N}$. Consider a function class indexed by $i \in \mathbb{N}$ such that,
\begin{equation*}
    f_i(a) = \begin{cases}
    \frac{1}{2i} &\text{if } a = 0\\
    1 &\text{if } a = i\\
    0 &\text{o.w.}
    \end{cases}
\end{equation*}

It is clear that $\depth_{\epsilon, \delta'}^0(\F) = 1$ since the query action $0$ produces an observation $\frac{1}{2i}$ containing enough information to identify the optimal action.

We'll consider a companion empty problem,
\begin{equation*}
    f_0(a) =  0 ~\forall a \in \A
\end{equation*}
with zero mean unit Gaussian noise. Let's consider a (possibly randomized) algorithm $\Alg$ and its interaction with $f_0$ and $f_{j}$ over $n$ rounds. We will show that $\Alg$ cannot succeed at identifying an optimal arm with probability at least $1-\delta$ after $n$ steps for all $f \in \F$. We prove this by way of contradiction. Assume $m^\sigma_{\Alg}(\epsilon, \delta) \leq n$ for a finite natural number $n \in \mathbb{N}$. 

Let $k(n)$ be a random variable specifying $\Alg$'s guess for the optimal action at time $n$.  We allow for $\Alg$ to guess a distribution over candidate optimal actions at time $n$. In this case $k(n)$ captures the realization of a sample from this distribution. For any $j \in \mathbb{N}$ define the event $\mathcal{E}_j $ as,
\begin{equation*}
    \mathcal{E}_j = \{k(n) = j\}
\end{equation*}
When $m_\Alg^\sigma(\epsilon, \delta) \leq n$ it follows that for all $i \in \mathbb{N}$,
\begin{equation}\label{equation::lower_bound_probability_guessing_correctly}
    \mathbb{P}_{\Alg, f_i}( \mathcal{E}_i ) \geq 1-\delta.
\end{equation}
The divergence decomposition of bandit problems (Lemma~\ref{lemma::divergence_decomposition_bandits}) implies,
\begin{align*}
    \mathrm{KL}( \mathbb{P}_{\Alg, f_{0}}  \parallel \mathbb{P}_{\Alg, f_{j}} ) &= \mathbb{E}_{\Alg, f_0}[T_0(n) ] \left( \frac{1}{2j} \right)^2 +  \mathbb{E}_{\Alg, f_0}[T_j(n) ] 
\end{align*}
where $T_m(n)$ equals the random variable equal to the number of times algorithm $\Alg$ tried action $m$ up to time $n$. Let $i\in\mathbb{N}$ be such that $i\geq 16$ and $i\geq 8n$. Order the indices $[i, 2i-1]$ as $\mathrm{I}_1,\cdots, \mathrm{I}_i$ such that $\mathbb{E}_{\Alg, f_0}[T_{\mathrm{I}_1}(n)] \leq \cdots \leq \mathbb{E}_{\Alg, f_0}[T_{\mathrm{I}_i}(n)]$. It follows that for all $j \in [1, \cdots, i/2]$, 
\begin{equation*}
    \mathbb{E}_{\Alg, f_0}[T_{\mathrm{I}_j}(n)] \leq 2n/i 
\end{equation*}
This is because if this was not the case, then we would have $\sum_{j = i/2+1}^{i} \mathbb{E}_{\Alg, f_0}[T_{\mathrm{I}_j}(n)] > i/2 * 2n/i = n$, a contradiction. 

Now observe that $\mathbb{E}_{\Alg, f_0}[T_0(n)] \leq n$, and by asssumption $n \leq i/8$. Therefore for all $j \in [1,\cdots, i/2]$:
\begin{align*}
     \mathrm{KL}( \mathbb{P}_{\Alg, f_{0}} \parallel \mathbb{P}_{\Alg, f_{\mathrm{I}_j}} ) &= \mathbb{E}_{\Alg, f_0}[T_0(n) ] \left( \frac{1}{2I_j} \right)^2 +  \mathbb{E}_{\Alg, f_0}[T_j(n) ] \\
     &\stackrel{(i)}{\leq} \frac{n}{4i^2} +  \frac{1}{4}  \\
     &\leq  1/4 + 1/32 \\
     &= 9/32 .
\end{align*}
Where inequality $(i)$ follows because $I_j \in [ i, \cdots, 2i-1]$. Recall that $k(n)$ is a random variable specifying $\Alg$'s guess for the optimal action at time $n$ and $\mathcal{E}_i$ is the event that $\Alg$ guessed action $k(n)=i$ as optimal at time $n$.
Let $\hat{i}$ denote the action such that, 
\begin{equation*}
    \hat{j} = \argmin_{j \in [1, \cdots, i/2]} \mathbb{P}_{\Alg, f_0}\left(\mathcal{E}_{\mathrm{I}_j} \right) 
\end{equation*}
Since $\sum_{j =1}^\infty \mathbb{P}_{\Alg, f_0}\left(\mathcal{E}_j\right) = 1$ it follows that $\sum_{j=1}^{i/2} \mathbb{P}_{\Alg, f_0}\left(\mathcal{E}_{\mathrm{I}_j} \right) \leq 1$ and therefore,
\begin{equation*}
    \mathbb{P}_{\Alg, f_0}\left(\mathcal{E}_{\mathrm{I}_{\hat{j}}} \right)  \leq 2/i
\end{equation*}

Pinsker's inequality implies,
 \begin{equation*}
    2\left(  \mathbb{P}_{\Alg, f_0}\left(\mathcal{E}_{\mathrm{I}_{\hat{j}}} \right) -  \mathbb{P}_{\Alg, f_{\mathrm{I}_{\hat{j}}}}\left(\mathcal{E}_{\mathrm{I}_{\hat{j}}} \right)  \right)^2 \leq \mathrm{KL}( \mathbb{P}_{\Alg, f_{0}}  \parallel \mathbb{P}_{\Alg, f_{\mathrm{I}_{\hat{j}}}} )  \leq 9/32.
 \end{equation*}
this in turn implies,
\begin{equation*}
    \left|  \mathbb{P}_{\Alg, f_0}\left(\mathcal{E}_{\mathrm{I}_{\hat{j}}} \right) -  \mathbb{P}_{\Alg, f_{\mathrm{I}_{\hat{j}}}}\left(\mathcal{E}_{\mathrm{I}_{\hat{j}}} \right)  \right| \leq 3/8.
 \end{equation*}
Finally, since $\mathbb{P}_{\Alg, f_0}\left(\mathcal{E}_{\mathrm{I}_{\hat{j}}} \right)  \leq 2/i$ we conclude that,

\begin{equation*}
\mathbb{P}_{\Alg, f_{\mathrm{I}_{\hat{j}}}}\left(\mathcal{E}_{\mathrm{I}_{\hat{j}}} \right) \leq     \frac{2}{i} + \left|  \mathbb{P}_{\Alg, f_0}\left(\mathcal{E}_{\mathrm{I}_{\hat{j}}} \right) -  \mathbb{P}_{\Alg, f_{\mathrm{I}_{\hat{j}}}}\left(\mathcal{E}_{\mathrm{I}_{\hat{j}}} \right)  \right| \leq 3/8 + \frac{2}{i}.
 \end{equation*}

Since $i > 16$ it follows that,

\begin{equation*}
\mathbb{P}_{\Alg, f_{\mathrm{I}_{\hat{j}}}}\left(\mathcal{E}_{\mathrm{I}_{\hat{j}}} \right) < 1/2
 \end{equation*}
thus contradicting Equation~\ref{equation::lower_bound_probability_guessing_correctly} when $\delta < 1/2$.

\subsection{Proof of Theorem~\ref{theorem::noisy_depth_lower_bounds_det_depth}}

In the proof of Theorem~\ref{theorem::noisy_depth_lower_bounds_det_depth} we define a family of function classes indexed by $K \in \mathbb{N}$. We split the proof of our main result by first showing in Lemma~\ref{lemma::upper_bound_depth_connection_hann} an upper bound on the query complexity of these function classes, and second proving a matching lower bound in Lemma~\ref{lemma::lower_bound_depth_connection_hann}.

\begin{lemma}\label{lemma::upper_bound_depth_connection_hann}
There exists a universal constant $c > 0$ such that for any $K \in \mathbb{N}$ satisfying $K \geq 2$ the function class $\F$ with action space $\A = \{ a_i \}_{i=0}^K
$ defined as $\F = \{ f_i \}_{i=1}^K$ with,
\begin{equation*}
    f_i(a) = \begin{cases}
            \frac{i}{4 K} &\text{ if } a  = a_0\\
            1/2&\text{if } a \neq i\\
            1 &\text{o.w.}
    \end{cases}
\end{equation*}
satisfies that for any $\epsilon \in [0,1/2)$, if $\sigma^2 > \frac{1}{c\log^{2/3}(K) K^{2/3}}$,
\begin{equation*}
\depth_{\epsilon, 1/4}^\sigma(\F) \leq c \log^{2/3}(K) K^{2/3}\sigma^2,
\end{equation*}
and if  $\sigma^2 \leq \frac{1}{c\log^{2/3}(K) K^{2/3}}$ then,
$$
\depth_{\epsilon, 1/4}^\sigma(\F)  = \depth_{\epsilon, 0}^0(\F) = 1.
$$

\end{lemma}
\begin{proof}

 \paragraph{Upper bound}

To prove the desired bound we consider the following algorithm,
\begin{enumerate}
    \item Estimate the mean reward of action $a_0$ up to $0 < \alpha \leq 1$ accuracy. We will choose an appropriate value of $\alpha$ at the end of the proof. 
    \item Use the estimator from 1) to identify a candidate set of $2\alpha K$ functions that agree with the observed data.  
\end{enumerate}

Lemma~\ref{lemma::gaussian_single_concentration} implies step 1. can be achieved with an error probability of at most $\delta_0$ by constructing an empirical mean estimator $\hat{\mu}_0$ of the mean reward of action $a_0$ using at most  $\frac{2\sigma^2\log(2/\delta_0)}{\alpha^2}$ samples. 

In step $2$ we define a candidate set of models $\mathcal{I} = \{ i \text{ s.t. } |\hat{\mu}_0 - f_i(a_0)| \leq \alpha \}$ that agree with the estimator $\hat{\mu}_0$. The size of $\tilde{I}$ satisfies  $|\tilde{I}| \leq  2\alpha K$. 

We will analyze two cases:
\begin{enumerate}
    \item $\alpha <1/K$ so that $| \mathcal{I}| \leq 2\alpha K < 2$.
    \item $\alpha \geq 1/K$. 
\end{enumerate}

\textbf{Case 1}: when $\alpha < 1/K$ and $2\alpha K < 2$ the set $\mathcal{I}$ satisfies $| \mathcal{I}| = 1$ and therefore with probability at least $1-\delta_0$ the algorithm can identify the optimal action. In this case it is sufficient to set $\delta_0 = 1/4$ and $\alpha = \frac{1}{2K}$ and the number of queries is upper bounded by $\max(8 \log(8) \sigma^2 K^2,1)$.

\textbf{Case 2}: when $\alpha \geq 1/K$ the algorithm then estimates each action in $\mathcal{I} $ up to $1/4$-accuracy ensuring an error probability per action of at most $\delta$. Lemma~\ref{lemma::gaussian_single_concentration} implies this can be done by using $32 \sigma^2 \log(2/\delta)$ samples for each of the actions in $\mathcal{I}$.  

The union bound implies that in this case the action with the greatest empirical mean reward is guaranteed to be the best action with an error probability of error at most $\delta_0 + 2\alpha K \delta$. 

Setting $\delta_0 = 1/8$ and $\delta = \frac{1}{16\alpha K}$ we conclude the algorithm can achieve an error probability of at most $1/4$ with a total number of queries upper bounded by,
\begin{align*}
    \frac{2\sigma^2\log(2/\delta_0)}{\alpha^2} + 2\alpha K \cdot 32 \sigma^2 \log(2/\delta) &=  \frac{2\sigma^2\log(16)}{\alpha^2} + 2\alpha K \cdot 32 \sigma^2 \log(32\alpha K) \\
    &\leq \frac{2\sigma^2\log(16)}{\alpha^2} + 2\alpha K \cdot 32 \sigma^2 \log(32 K).
\end{align*}
If the algorithm uses $\alpha = \left(\frac{\log(16)}{32 K\log(32K)}\right)^{1/3}$ its query complexity can be upper bounded by,
\begin{equation*}
\max\left(2 \cdot \frac{\log^{1/3}(16)}{(32)^{1/3} }  \log^{2/3}(32K)  K^{2/3}\sigma^2, 1\right)
\end{equation*}
Up to $\mathcal{O}(\cdot)$ notation Case 2's query complexity is smaller than the query upper bound from Case 1.  We conclude there exists a universal constant $c > 0$ such that, 
\begin{equation*}
    \depth_{\epsilon,1/4}^\sigma(\F) \leq \begin{cases}
                    1 &\text{if } \sigma^2 \leq \frac{c}{\log^{2/3}(K) K^{2/3}}\\
                    c \log^{2/3}(K) K^{2/3}\sigma^2 &\text{o.w.}  
                \end{cases}
\end{equation*}

We conclude the proof by noting $\depth_{\epsilon,1/4}^\sigma(\F)$ must be at least $1$ (because otherwise the probability of error would be at least $1/2$) so that when $\sigma^2 \leq \frac{c}{\log^{2/3}(K) K^{2/3}}$ we conclude $\depth_{\epsilon,1/4}^\sigma(\F) = 1$.

\end{proof}

\begin{lemma}\label{lemma::lower_bound_depth_connection_hann}
For any $K \in \mathbb{N}$ satisfying $K \geq 2$ the function class $\F$ with action space $\A = \{ a_i \}_{i=0}^K
$ defined as $\F = \{ f_i \}_{i=1}^K$ with,
\begin{equation*}
    f_i(a) = \begin{cases}
            \frac{i}{4 K} &\text{ if } a  = a_0\\
            1/2&\text{if } a \neq i\\
            1 &\text{o.w.}
    \end{cases}
\end{equation*}
satisfies that for any $\epsilon \in [0,1/2)$,
\begin{equation*}
\depth_{\epsilon, 1/4}^\sigma(\F) \geq \max\left( \frac{1}{3} K^{2/3} \sigma^2, 1\right).
\end{equation*}

\end{lemma}

\begin{proof}
 \paragraph{Lower bound} In our proof we will use the helper function $\bar{f}_K$ defined as,
\begin{equation*}
\bar{f}_K(a) = \begin{cases}
    \frac{1}{4} &\text{if } a = a_0\\
    1/2 &\text{o.w.}
\end{cases}
\end{equation*}

Consider any algorithm $\Alg$ interacting with problems in $\F$ and with $\bar{f}_K$ over $n$ rounds. We will show in case $m_\Alg^\sigma(\epsilon, 1/4) \leq n$, then $n$ must be lower bounded by $\max\left( \frac{1}{3} K^{2/3} \sigma^2, 1\right)$. 

As we have argued in the proof of Lemma~\ref{lemma::upper_bound_depth_connection_hann} it is clear that any algorithm with an error probability of at most $1/4$ must make at least one single query, thus if $m_{\Alg}^\sigma(\epsilon, 1/4) \leq n$ then $n \geq 1$. To complete our result we want to show that $n \geq \frac{1}{3}K^{2/3}\sigma^2$ as well. We dedicate the rest of the argument to prove this lower bound.

Our argument is based on analyzing the hypothetical interaction between $\Alg$ and $\bar{f}_K$ and using it to derive properties of the interactions between $\Alg$ and functions in $\F$. Throughout this exploration we assume that $\Alg$ satisfies $m_\Alg^\sigma(\epsilon, 1/4) \leq n$ so that,
\begin{equation}\label{equation::error_prob_lower_bound_3fourths}
    \mathbb{P}_{\Alg, f_i}( k(n) = i )  \geq 3/4. 
\end{equation}
Where $k(n)$ is the random variable equal to $\Alg$'s guess for the optimal action at time $n$. Let's start by considering the expected number of queries of action $a_0$ by algorithm $\Alg$ when interacting with problem $\bar{f}_K$:
$$L = \mathbb{E}_{\Alg, \bar{f}_K }\left[ T_0(n) \right],$$  where $T_0(n)$ denotes the (random) number of queries of action $a_0$ by $\Alg$ up to time $n$ and the expectation is taken w.r.t the randomness of $\Alg$ and the measurement noise during its interaction with $\bar{f}_K$, Define $\mathcal{I} = \{ i \in [K] : | \bar{f}_K(a_0)  - f_i(a_0)| \leq \frac{\sigma}{4\sqrt{L}} \} $ be the set of indices of functions in $\F$ having $a_0$ function values close to $\bar{f}_K(a_0)= 1/4$. The size of $\mathcal{I}$ satisfies: %
\begin{equation*}
    |\mathcal{I} | \geq \frac{\sigma}{4\sqrt{L}}\cdot K.
\end{equation*}
 Consider the tuples of expectations and probabilities of the random variables $\T_i(n)$ and $\mathbf{1}(k(n)=i)$ for $i \in \mathcal{I}$ induced by $\Alg$'s interaction with $\bar{f}_K$, $$\{ ( \mathbb{E}_{\Alg, \bar{f}_K}\left[  T_i(n)  \right],   \mathbb{P}_{\Alg, \bar{f}_K}\left( k(n) = i \right))\}_{i\in \mathcal{I}} .$$ Lemma~\ref{lemma::supporting_lemma_min_sequences} implies there is an index $\tilde{i}$ such that,
\begin{align}
    \mathbb{E}_{\Alg, \bar{f}_K}\left[  T_{\tilde{i}}(n)  \right] &\leq \frac{3n}{|\mathcal{I}|} \leq \frac{12n\sqrt{L}}{K\sigma} \label{equation::expectation_tilde_upper}\\
     \mathbb{P}_{\Alg, \bar{f}_K}\left( k(n) = \tilde{i} \right) &\leq  \frac{3}{|\mathcal{I}|}\label{equation::probability_upper_bound}
\end{align}
Let's now compute the $\mathrm{KL}$ between $\mathbb{P}_{\Alg, \bar{f}_{K}}$ and $ \mathbb{P}_{\Alg, f_{\tilde{i}}}$. This quantity satisfies the following inequalities,
\begin{align}
    \mathrm{KL}\left(  \mathbb{P}_{\Alg, \bar{f}_{K}}  \parallel \mathbb{P}_{\Alg, f_{\tilde{i}}}  \right) &= \mathbb{E}_{\Alg, \bar{f}_K}\left[ T_0(n) \right] \left(f_{\tilde{i}}(a_0) - \bar{f}_K(a_0)\right)^2 \cdot \frac{1}{\sigma^2}+ \mathbb{E}_{\Alg, \bar{f}_K}\left[ T_{\tilde{i}}(n) \right] \cdot \frac{1}{4\sigma^2} \notag\\
     &\stackrel{(i)}{\leq}  \mathbb{E}_{\Alg, \bar{f}_K}\left[ T_0(n) \right] \left(\frac{\sigma}{4\sqrt{L}}\right)^2 \cdot \frac{1}{\sigma^2}+ \mathbb{E}_{\Alg, \bar{f}_K}\left[ T_{\tilde{i}}(n) \right] \cdot \frac{1}{4\sigma^2} \notag \\
     &\stackrel{(ii)}{\leq} \frac{1}{16} + \frac{6n\sqrt{L}}{K\sigma^3} \label{equation::kl_upper_bound_hann_support}
\end{align}
where inequality $(i)$ holds because $\tilde{i} \in \mathcal{I}$ and therefore $\left | f_{\tilde{i}}(a_0) - \bar{f}_K(a_0)\right |  \leq \frac{\sigma}{4\sqrt{L}}$ and $(ii)$ because $\mathbb{E}_{\Alg, \bar{f}_K}\left[ T_0(n) \right] = L$ and $\mathbb{E}_{\Alg, \bar{f}_K}\left[  T_{\tilde{i}}(n)  \right] \leq \frac{12n\sqrt{L}}{K\sigma} $ as implied by equation~\ref{equation::expectation_tilde_upper}.  In order to obtain a lower bound for $n$ we consider two cases. 
\begin{enumerate}
    \item Case 1. $\frac{\sigma}{4\sqrt{L}} \leq \frac{1}{6K}$.
        \item Case 2. $\frac{\sigma}{4\sqrt{L}} > \frac{1}{6K}$.
\end{enumerate}
In Case 1 we have $\sqrt{L} \geq \frac{3}{2}\sigma K$ and therefore  $n \geq L \geq \frac{9}{4} \sigma^2 K^2$. In Case 2 we have $|\mathcal{I} | \geq 6$ and therefore equation~\ref{equation::probability_upper_bound} implies $$ \mathbb{P}_{\Alg, \bar{f}_K}\left( k(n) = \tilde{i} \right) \leq \frac{3}{|\mathcal{I}|} \leq \frac{1}{2}.$$ 
Equation~\ref{equation::error_prob_lower_bound_3fourths} implies that $ \mathbb{P}_{\Alg, f_{\tilde{i}}}\left( k(n) = \tilde{i} \right) \geq \frac{3}{4}$. Combining the last two inequalities,
\begin{equation*}
 \frac{1}{4}    \leq     \left| \mathbb{P}_{\Alg, f_{\tilde{i}}}\left( k(n) = \tilde{i} \right) - \mathbb{P}_{\Alg, \bar{f}_K}\left( k(n) = \tilde{i} \right)\right|. 
\end{equation*}
The probability gap lower bound above can be combined with Pinsker inequality and inequality~\ref{equation::kl_upper_bound_hann_support} to obtain,
\begin{align*}
    \frac{1}{8} &\leq 2\left(    \mathbb{P}_{\Alg, \bar{f}_K}\left( k(n) = \tilde{i} \right)  - \mathbb{P}_{\Alg, f_{\tilde{i}}}\left( k(n) = \tilde{i} \right)   \right)^2 \\
    &\leq   \mathrm{KL}\left(  \mathbb{P}_{\Alg, \bar{f}_{K}}  \parallel \mathbb{P}_{\Alg, f_{\tilde{i}}}  \right)   \\
    &\leq \frac{1}{16} + \frac{6n\sqrt{L}}{K\sigma^3} \\
    &\leq \frac{1}{16} + \frac{6n^{3/2}}{K\sigma^3}
\end{align*}
And therefore,
\begin{equation*}
    n \geq \frac{1}{3} K^{2/3} \sigma^2.
\end{equation*}
This finalizes the lower bound showing that,
\begin{equation*}
    \depth^\sigma_{\epsilon, 1/4}(\F) \geq \frac{1}{3} K^{2/3} \sigma^2.
\end{equation*}

\end{proof}

\theoremnoisydepthann*

\begin{proof}
In order to prove Theorem~\ref{theorem::main_theorem_result_noisy_separation} 
we use function classes $\F$ and action spaces $\A$ defined by $K \in \mathbb{N}$ and $\epsilon' \in [0,1/2)$ such that $K \geq 2$, $\epsilon' > \epsilon$ and where
\begin{equation*}
\A = \{ a_i \}_{i=0}^K
\end{equation*}
 All elements in $\F$ are indexed by $i \in [K]$ such that,
\begin{equation*}
    f_i(a) = \begin{cases}
            \frac{i}{4 K} &\text{ if } a  = a_0\\
            1/2&\text{if } a \neq i\\
            1 &\text{o.w.}
    \end{cases}
\end{equation*}
We assume the noise model satisfies $\xi \sim \mathcal{N}(0, \sigma^2)$. We start by showing that $\gamma_{F, \epsilon } = \frac{1}{K} $ for all $\epsilon < 1/2$. By definition,
\begin{equation}
    \begin{aligned}
\gamma_{\mathcal{F}_K,\epsilon}=\sup_{p \in \Delta(\A)} \inf_{f_i \in \mathcal{F}_K} \mathbb{P}_{a \sim p}\left(f_i(a_i)-f_i(a)\leq \epsilon\right),
    \end{aligned}
\end{equation}
where $\Delta(\A)$ is the set of all distributions on $\A$. This is because when $0 \leq \epsilon < 1/2$ action $a_i$ is the unique $\epsilon$-optimal action for problem $f_i$. Any distribution $p$ over $\A$ satisfies $$\min_{i \in [K]} p(a_i) \leq \frac{1}{K}.$$
Thus it follows that,
\begin{equation*}
    \mathbb{P}_{a \sim p}\left(f_i(a_i)-f_i(a)\leq \epsilon\right) = \min_{i \in [K]} p(a_i) \leq \frac{1}{K}.
\end{equation*}
for any $p \in \Delta(\A)$. And therefore, 
\begin{equation*}
    \gamma_{\F, \epsilon} \leq \frac{1}{K}. 
\end{equation*}
Finally, the uniform distribution $p = \mathrm{Uniform}(a_1, \cdots, a_K ) $ satisfies $$\inf_{f_i \in \F}  \mathbb{P}_{a \sim p}\left(f_i(a_i)-f_i(a)\leq \epsilon\right) = \frac{1}{K},$$ 
showing that $\gamma_{\F,\epsilon} = 1/K$. The remainder of this Theorem follows from the results of Lemmas~\ref{lemma::upper_bound_depth_connection_hann} and~\ref{lemma::lower_bound_depth_connection_hann}. 

\end{proof}

\subsection{Proof of Theorem~\ref{thm:noisy_depth_is_det_depth_for_low_noise}}\label{appendix::gap_proof}

\begin{definition}[\bf Gap of $\F$]\label{def:gap_F}
Let $\F \subseteq [0,1]^\A$.  Let $\epsilon,\delta \in [0,1]$ be a finite function class over a finite action space, and $\Alg$ be any (possibly randomized) algorithm. Let $\Gamma_n(\Alg)$ denote all length $n$ partial action-rewards trajectories $((a_1,r_1),...,(a_n,r_n))$ in the support of the trajectory distribution of $\Alg$. If there are no such trajectories (for example when $n$ is larger than the largest trajectory in the support of $\Alg$) we define $\Gamma_n(\Alg) = \emptyset$. For any $n$ we define the $n$-degree {\it Gap} of $\Alg$ and $\F$ to be:
$$
\mathrm{Gap}_n(\Alg, \F) = \inf_{\tau_n \in \Gamma_n(\Alg),~ r \in \{f(a_n)  | f \in \F(\tau_{n-1} )\} \text{ s.t. } r \neq r_n}  ~|r-r_n|. 
$$
Where $\tau_i = ((a_1,r_1),...,(a_i,r_i))$ for all $i \in [n]$ and $\tau_0 = \emptyset$. When $\Gamma_n(\Alg) = \emptyset$ we define $\mathrm{Gap}_n(\Alg, \F) = \infty$.  We define the {\it Gap} of $\Alg$ and $\F$ to be:
$$
\mathrm{Gap}(\Alg, \F) = \min_{n \in \mathbb{N}}  \mathrm{Gap}_n(\Alg, \F).
$$
Finally, we define the {\it Gap} of $\F$ to be
\begin{equation*}
  \mathrm{Gap}_{\epsilon, \delta}( \F) =   \max_{ \Alg \in \mathbb{A}_{\epsilon, \delta}} \mathrm{Gap}(\Alg, \F) 
\end{equation*}
where $\mathbb{A}_{\epsilon,\delta}$ denotes the set of randomized algorithms with query complexity $\depth^0_{\epsilon, \delta}(\F)$.

\end{definition}

\theoremgapnoisynoiseless*

\begin{proof} We prove this result by constructing an algorithm $\Alg$ that satisfies $m_\Alg^\sigma(\epsilon, \delta)\leq \depth^0_{\epsilon, \delta'}(\F)$.

Lemma~\ref{lemma::gaussian_single_concentration} implies that when noise variance satisfies $\sigma^2 < \frac{\mathrm{Gap}^2_{\epsilon, \delta'}( \F) }{4\log(2\depth^0_{\epsilon, \delta'}(\F)/(\delta-\delta'))}$ then a single noisy query of action $a \in \A$ for function $f \in \F$ produces a value $\hat r$ satisfying $|\hat r - f(a)| > \frac{\mathrm{Gap}_{\epsilon, \delta'}( \F) }{2}$ with probability at most $\frac{\delta-\delta'}{\depth^0_{\epsilon, \delta'}(\F)}$. 

Pick an algorithm $\widetilde{\Alg}$ realizing $\depth^0_{\epsilon, \delta'}(\F)$ and  $\mathrm{Gap}_{\epsilon, \delta'}( \F) $, i.e. such that $$\widetilde{\Alg} = \argmax_{ \Alg \in \mathbb{A}_{\epsilon, \delta'}} \mathrm{Gap}(\Alg, \F) $$ where $\mathbb{A}_{\epsilon, \delta'}$ is the set of randomized algorithms with query complexity $\depth^0_{\epsilon, \delta'}(\F)$.  We will now construct a new algorithm $\Alg$ based on $\widetilde{\Alg}$. The difference will be that $\Alg$ will map all received noisy rewards to a candidate noiseless reward. The resulting candidate noiseless reward is used to build a reconstructed noiseless action-reward history that is fed into $\widetilde{\Alg}$.

Algorithm $\Alg$ plays actions $a_1, \cdots, a_{i}$, observes rewards $\hat r_1, \cdots, \hat r_{i}$ and adjusts them to fit noiseless values $ r_1, \cdots,  r_{i}$ to build a partial history  $\tau_{i} = ((a_1, r_1), \cdots, (a_{i}, r_{i}))$, where we define $\tau_0 = \emptyset$. We will now explain how $\Alg$ constructs $r_1, \cdots, r_i$ from $((a_1, \hat{r}_1), \cdots, (a_i, \hat{r}_i))$ below. This partial trajectory is fed into $\widetilde{\Alg}$. If $\widetilde{\Alg}$ is ready to suggest an $\epsilon$-optimal policy, $\Alg$ outputs it as its own guess. In case $\widetilde{\Alg}$ decides to proceed exploring it will propose a distribution over actions $ \widetilde{\Alg}(\tau_i)$. Algorithm $\Alg$ will play $a_{i+1} \sim \widetilde{\Alg}(\tau_i)$ and observe a noisy reward $\hat{r}_{i+1}$.

It remains to define how $\Alg$ after constructing a noiseless trajectory $((a_1, r_1), \cdots, (a_{i-1}, r_{i-1}))$ will map a noisy observation $\hat{r}_{i}$ when playing action $a_{i}$ to a candidate noiseless reward value $r_{i}$. This is done by solving the nearest neighbors problem,
$$r_{i} = \argmin_{r 
\in \{f(a_{i})  | f \in \F(\tau_{i-1} )\}} |\hat r_{i} - r|.$$

When $\sigma^2 <  \frac{\mathrm{Gap}^2_{\epsilon, \delta'}( \F) }{4\log(2\depth^0_{\epsilon, \delta'}(\F)/(\delta-\delta'))}$ it follows that when $\Alg$ interacts with function $f \in \F$, 
\begin{equation}\label{equation::upper_bound_error_alg}
    \mathbb{P}_{\Alg, f}\left( r_i = f( a_i ) ~| ~r_j = f(a_j) ~~\forall j \leq {i-1}  \text{ and } \Alg \text{ plays } a_i = a\right) \geq 1 -  \frac{\delta - \delta'}{\depth^0_{\epsilon,\delta'}(\F)}.
\end{equation}
This is because when $r_j = f(a_j)$ for all $j \leq i-1$, then the partial trajectory $((a_1, r_1), \cdots, (a_{i-1}, r_{i-1}))$ is in the support of $\widetilde{\Alg}$. Integrating equation~\ref{equation::upper_bound_error_alg} over the possibly random choice of $a_i$,
\begin{equation*}
    \mathbb{P}_{\Alg, f}\left( r_i = f( a_i ) ~| ~r_j = f(a_j) ~~\forall j \leq {i-1}  \right) \geq 1- \frac{\delta - \delta'}{\depth^0_{\epsilon,\delta'}(\F)} 
\end{equation*}

Let's consider the event $\mathcal{E}_{\mathrm{bound}}$ when $\Alg$ outputs a guess for an $\epsilon$-optimal policy after at most $\depth_{\epsilon, \delta'}(\F)$ queries. Let also define $\mathcal{E}_{\mathrm{nn}}$ to be the event that all nearest neighbor estimators $r_i$ during $\Alg$'s execution match the correct mean reward.  Additionally define $\mathcal{E}_{\mathrm{correct}} $ as the event when $\Alg$ outputs a correct $\epsilon$-optimal policy. A change of measure argument implies that, for any $f$,
\begin{align*}
    \mathbb{P}_{\Alg, f}(\mathcal{E}_{\mathrm{correct}} \cap \mathcal{E}_{\mathrm{bound}}\cap \mathcal{E}_{\mathrm{nn}}) &\geq     \mathbb{P}_{\widetilde{\Alg}, f}(\mathcal{E}_{\mathrm{correct}} \cap \mathcal{E}_{\mathrm{bound}})\cdot \left( 1- \frac{\delta - \delta'}{\depth^0_{\epsilon,\delta'}(\F)}\right)^{\depth^0_{\epsilon,\delta'}(\F)} \\
    &\stackrel{(i)}{\geq} \mathbb{P}_{\widetilde{\Alg}, f}(\mathcal{E}_{\mathrm{correct}}\cap \mathcal{E}_{\mathrm{bound}})\cdot \left( 1- (\delta - \delta')\right)
\end{align*}
inequality $(i)$ holds because for $a \in [0,1]$ and $n \in \mathbb{N}$ the bernoulli inequality implies $(1-\frac{a}{n})^n \geq 1-a$.  Finally, $\mathbb{P}_{\widetilde{\Alg}, f}\left( \mathcal{E}_{\mathrm{correct}} \cap \mathcal{E}_{\mathrm{bound}} \right) \geq 1-\delta'$ since $\widetilde{\Alg}$ realizes $\depth^0_{\epsilon, \delta'}(\F)$. Combining these inequalities we obtain,
\begin{align*}
    \mathbb{P}_{\Alg, f}(\mathcal{E}_{\mathrm{correct}} \cap \mathcal{E}_{\mathrm{bound}}\cap \mathcal{E}_{\mathrm{nn}}) &\geq    (1-\delta')\cdot \left( 1- (\delta - \delta')\right)\\
    &= 1-\delta +   \delta' \delta - (\delta') ^2\\
    &> 1-\delta.
\end{align*}
Since $\mathbb{P}_{\Alg, f}(\mathcal{E}_{\mathrm{correct}} \cap \mathcal{E}_{\mathrm{bound}} ) \geq \mathbb{P}_{\Alg, f}(\mathcal{E}_{\mathrm{correct}} \cap \mathcal{E}_{\mathrm{bound}}\cap \mathcal{E}_{\mathrm{nn}})$ we conclude $\Alg$'s probability of error when interacting with any function $f$ for at most $\depth^0_{\epsilon, \delta'}(\F)$ rounds is upper bounded by $\delta$. This in turn implies that,
\begin{equation*}
    \depth_{\epsilon, \delta}^{\sigma}( \F ) \leq  \depth^0_{\epsilon, \delta'}(\F).
\end{equation*}

\end{proof}

\section{Missing discussion and proofs of \Cref{sec:sep}}

\subsection{Regret vs. QC: noise-free case}\label{app::noise_free_discussion}

In this section we prove there is a a separation between regret and $\epsilon$-optimality in the noise-free setting. In noise-free problems with finite action spaces, it is possible to achieve finite regret. Any algorithm that is guaranteed to find an optimal action in finitely many queries can find the optimal action and then incur zero regret in any subsequent time-step by playing it. In Theorem~\ref{theorem::main_separation_noiseless} we show there are problems where the objective of finding an optimal action cannot be achieved without paying a regret scaling linearly with the query complexity.

\begin{restatable}{theorem}{theoremmainseparationnoiseless}\label{theorem::main_separation_noiseless}
Let $d \in \mathbb{N}$, $\epsilon \in [0,1)$ and $\gamma \in (0,1)$ such that $\epsilon + \gamma <1$. There is a function class $\F\subseteq [0,1]^\A$  such that $\depth_{\epsilon, 0}^0(\F) =d$, and such that for any $T \ge d$ and
any algorithm $\Alg$ such that $m^0_\Alg(\epsilon, 0) \leq T$ satisfies $$\max_{f \in \F} \mathrm{Regret}_{\Alg}(T,f) \geq d.$$
Moreover for any $T$ there is an algorithm $\Alg'$ such that  $\max_{f \in \F}\mathrm{Regret}_{\Alg'}(T, f) \leq (\epsilon + \gamma)T$.
\end{restatable}

The function class $\F$ used to prove Theorem~\ref{theorem::main_separation_noiseless} is based on a ``breadcrums trail'' construction where the action space can be identified with a binary tree of depth $d+1$. The function class is indexed by the $2^d$ leaf actions. The function identified by a leaf action $a_{\mathrm{leaf}}$ achieves an optimal value of $1$ at $a_{\mathrm{leaf}}$. The actions along the path from the root to $a_{\mathrm{leaf}}$ have rewards of $0$ except the root action that achieves a value of $1-\epsilon-\gamma$. All leaf actions different from $a_{\mathrm{leaf}}$ have a value of zero and every other action has a reward of $1-\epsilon-\gamma$. It is clear that the algorithm that always plays the root action achieves a regret of order $(\epsilon + \gamma)T$. We prove the lower bound by proving, first that $\depth_{\epsilon, 0}^0(\F) = d$ and by showing that any algorithm achieving $\depth_{\epsilon, 0}^0(\F)$ must have necessarily selected at least $d$ actions with zero rewards when interacting with one of the functions of the class thus incurring regret at least $d$ despite the existence of a low regret algorithm. %

In order to prove Theorem~\ref{theorem::main_separation_noiseless} we introduce a family of function classes parameterized by $d \in \mathbb{N}$ and $\Delta \in [0,1]$ such that $\Delta = \epsilon + \gamma$. We define the function class $\F_{\Delta}$ over action space  $ \A_{\mathrm{tree}}$ indexed by the nodes of a height $d+1$ binary tree defined here as having $d+1$ levels and $2^{d+1}-1$ nodes. 
\begin{wrapfigure}{l}{0.5\textwidth}
\vspace{-.3cm}
\begin{tikzpicture}[level distance=1.5cm,
  level 1/.style={sibling distance=3cm},
  level 2/.style={sibling distance=1.5cm}]
  \node {$a_{1,1}$}
    child {node {$a_{2,1}$}
      child {node {$a_{3,1}$}}
      child {node {$a_{3,2}$}}
    }
    child {node {$a_{2,2}$}
    child {node {$a_{3,3}$}}
      child {node {$a_{3,4}$}}
    };
\end{tikzpicture}
\caption{$\A_{\mathrm{tree}}$ action set when $d = 2$.} \label{fig::binary_tree}
\vspace{-.5cm}
\end{wrapfigure}
We call $a_{l, i}$ the $i$-th action of the $l$-th level of the tree. See Figure~\ref{fig::binary_tree} to see an example of a tree with depth $3$ (in this case $d=2$).

We define the function class $\F_{\Delta} $ as indexed by paths from the root node in $\A_{\mathrm{tree}}$ to a leaf (alternatively indexed by leaves). For any such path $\mathbf{p} = \{ a_{1,1}, a_{2, i_2}, \cdots, a_{d+1, i_{d+1}}\}$ the function $f^{(\mathbf{p})}$ equals,

\begin{equation}\label{equation::tree_func_class_definition}
    f^{(\mathbf{p})}(a) = \begin{cases}
    1 &\text{if } a = a_{d+1, i_{d+1}}\\
    0& \text{if } a = a_{d+1, j} \text{ and } j \neq i_{d+1}
\\ 
0 &\text{if } a \in \{ a_{2,1}, a_{2, i_2}, \cdots, a_{d, i_{d}}\}  \\
    1-\Delta &\text{o.w.}
    \end{cases}
\end{equation}
The function parameterized by path $\mathbf{p}$ achieves an optimal value of one at the leaf of path $\mathbf{p}$, a value of zero at any other leaf different from $a_{d+1, i_{d+1}}$, the endpoint of $\mathbf{p}$ and also a value of zero at any other action in $\mathbf{p}$ different from the endpoint and the root $a_{1,1}$. The function also achieves a ``large'' reward value of $1-\Delta$ at the root $a_{1,1}$ and every other action not listed above. When $\epsilon < \Delta$, the only $\epsilon$-optimal action in $f^{(\mathbf{p})}$ equals $a_{d+1, i_{d+1}}$. Our first partial result towards proving Theorem~\ref{theorem::main_separation_noiseless} is to show the query complexity of $\F_\Delta$ equals $d$:

\begin{restatable}{proposition}{propositiondepthdtree}\label{proposition::depth_d_tree}
Let $\epsilon, \Delta \in [0,1]$ such that $\epsilon < \Delta$ and define $\F_\Delta$ as the tree function class defined in equation~\ref{equation::tree_func_class_definition}. The query complexity of $\F_\Delta$ satisfies,
\begin{equation*}
    \depth^{0}_{\epsilon, 0}( \F_\Delta ) = d 
\end{equation*}  
\end{restatable}

\begin{proof}
For any algorithm $\Alg$ its query complexity and depth satisfy $$\depth^0_{\epsilon, 0}(\F_\Delta) \leq m^0_{\Alg}(\epsilon, 0).$$ 
Consider an algorithm that queries the actions in $\A_{\mathrm{tree}}$ starting at $a_{2,1}$. If it observes a $0$ it continues querying its left-side leaf. If it is querying an interior node and observes a $1-\Delta$, it continues querying the left-leaf of its sibling action\footnote{Here we define the sibling of an action to be that which is the opposite child of its parent action.}. Algorithm $\Alg$ is guaranteed to have identified an action with a reward of $1$ after at most $d$ queries. The algorithm may not have queried an action of reward $1$ after exactly $d$ queries but it will be certain of its identity. 

Moreover, we can show $\depth^0_{\epsilon, 0}(\F_\Delta) \geq d$. 

Let $\Alg$ be an arbitrary deterministic algorithm. To prove that $\depth^0_{\epsilon, 0}(\F_\Delta) \geq d$ we start by noting that any action in $\A_\mathrm{tree}$ has exactly two possible values $(1-\Delta, 0)$ for inner actions and $(1,0)$ for leaf actions. 

In order to arrive at a contradiction let's assume $\Alg$ is able to identify an $\epsilon$-optimal action after only $d-1$ interactions for all $f \in \F_\Delta$. We'll prove that when limiting the interactions to only $d-1$ or less, there will be two functions in $\F_\Delta$ that will produce the same interaction trace.

This is easy to deduce. First observe that histories of size at most $d-1$ give rise to at most $2^{d-1}$ distinct reward traces. This is because any query action can take at most two reward values. Finally, since the total number of functions in $\F_\Delta$ is $2^{d}$ we conclude there must exist two functions that produce the same trace if these are limited to sizes at most $d-1$.

This finalizes the proof. 

\end{proof}

Our second result establishes that any algorithm $\Alg$ that is able to find an $\epsilon$-optimal action for $\epsilon < \Delta$ must incur in at least $d$ regret.

\begin{restatable}{lemma}{lemmaregretlowerboundtreeclass}\label{lemma::regret_lower_bound_tree_class}
Let $0 \leq \epsilon < \Delta$. For any algorithm $\Alg$ interacting with $(\F_\Delta, \A_{\mathrm{tree}})$ such that $m^0_{\Alg}(\epsilon, 0) \leq T$ there is a problem $f^{\mathbf{p}}$ such that when $\Alg$ interacts with it,
\begin{equation*}
    \regret_{\Alg}(T, f^{(\mathbf{p})}) \geq d. 
\end{equation*}
\end{restatable}
 \begin{proof}
Let $n_\mathbf{p}$ be the number of queries when $\Alg$ received a reward of $0$ when interacting with $f^{(\mathbf{p})}$ over $T$ rounds.

The algorithm's regret is lower bounded by $n_{\mathrm{p}}$. We proceed to show that when $m^0_{\Alg}(\epsilon, 0)\leq T$   

$$\max_\mathbf{p} n_\mathbf{p} \geq d$$.

Algorithm $\Alg$ can be understood as producing an action given a partial history. Given a partial history, $a_1, r_1, \cdots, a_t, r_t$ algorithm $\Alg$ produces an action $a_{t+1}$. We say $a_{t+1}$ is a dummy action if for all $f \in \F(a_1, r_1, \cdots, a_t,r_t)$ the value $f(a_{t+1})$ is the same. That is, action $a_{t+1}$ is uninformative because no matter what function $f$ algorithm $\Alg$ may be interacting with, this action will not have any useful information for $\Alg$. We say that an algorithm is in a ``reduced'' form if it never proposes a to play a dummy action.

Given any deterministic algorithm $\Alg$ we construct a ``reduced'' version of $\Alg$, which we call $\widetilde{\Alg}$. This algorithm can be constructed by skipping all dummy actions proposed by $\Alg$.

We work now with $\widetilde{\Alg}$, the reduced version of $\Alg$. Notice that irrespective of the partial history (say $a_1, r_1, \cdots, a_t, r_t$), either $|\F(a_1, r_1, \cdots, a_t, r_t)| = 1$ or when action $a_{t+1}$ is proposed by $\widetilde{\Alg}$ in reaction to this partial history, there is a function $f^{(\mathbf{p})} \in \F(a_1, r_1, \cdots, a_t, r_t)$ such that $f^{(\mathbf{p})}(a_{t+1}) = 0$. This is because for any non-dummy action there must exist two functions in the version space that have different values and any action in $\A_{\mathrm{tree}}$ can achieve two values over all functions in in $\F_{\Delta}$, one of which is always zero.

When an action with value $0$ is observed, the version space at most halves. that is if $a_{t+1}$ is a non-dummy action for history $a_1,r_1, \cdots, r_t, a_t$ then, 
\begin{equation}\label{equation::version_space_halved}
\F(a_1, r_1, \cdots, a_t, r_t) \leq 2\F(a_1, r_1, \cdots, a_t, r_t, a_{t+1}, 0 ).
\end{equation}

Finally, let's follow the sequence of actions $\widetilde{\Alg}$ proposes when all observed rewards are $0$. Let $a_1, 0, a_2, 0, \cdots$ be the resulting history. Equation~\ref{equation::version_space_halved} implies that $$\F(a_1, 0, \cdots, a_t, 0) \leq 2 \F(a_1, 0, \cdots, a_{t+1},0)$$ and therefore $\F(a_1, 0, \cdots, a_t, 0) \geq \F_\Delta/2^t$. Since  $\Alg$, can guess a correct an $\epsilon$-optimal action only whn the version space is of size $1$, we conclude that the history generated by $\widetilde{\Alg}$ when all observed rewards are $0$ must produce at least $d$ consecutive $0$ observations. Let $f^{(\widetilde{\mathbf{p}})}$ be the function that generates this ``zero'' history. It follows that $n_{\widetilde{\mathbf{p}}} \geq d$. This has to be because the algorithm is assumed to produce a valid $\epsilon$-optimal action at the end of the interaction with any function up to $T$ steps. 

This finalizes the desired result.

 \end{proof}

We have the necessary ingredients to prove Theorem~\ref{theorem::main_separation_noiseless}.  

\begin{proof} Theorem~\ref{theorem::main_separation_noiseless} is a corollary of Lemma~\ref{lemma::regret_lower_bound_tree_class}. The function class in Theorem~\ref{theorem::main_separation_noiseless} can be identified with $(\F_\Delta, \A_{\mathrm{tree}})$ with $\Delta > \epsilon$. Setting $\gamma = \Delta - \epsilon$ the algorithm $\Alg'$ that always selects action $a_{1,1}$ satisfies $\mathrm{Regret}_{\Alg'}(T, f) = (\epsilon +\gamma) T$ for all $T \in \mathbb{N}$ and all $f \in \F_\Delta$. 
\end{proof}

\subsection{Proof of Theorem~\ref{theorem::main_theorem_result_noisy_separation}}\label{section::proof_theorem_separation_reg_noisy}

In order to prove Theorem~\ref{theorem::main_theorem_result_noisy_separation} 
we use function classes $\F_K$ and action spaces $\A_K$ defined by $K \in \mathbb{N}$ and $\epsilon_1, \epsilon_2 \in [0,1]$ where
\begin{equation*}
\A_K = \A_1 \cup \A_2
\end{equation*}
such that $\A_1 = \{a^{(1)}_1, \cdots, a^{(1)}_{\lceil \log(K)\rceil}\}$ and $\A_2 = \{ a_1^{(2)}, \cdots, a_K^{(2)}\}$. All elements in $\F_K$ are indexed by $k \in [K]$ such that,
\begin{equation*}
    f_k(a) = \begin{cases}
            1/2 \pm \epsilon_1 \text{ if } a \in \A_1\\
            1-\epsilon_2&\text{if } a \neq a_k^{(2)}\\
            1 &\text{o.w.}
    \end{cases}
\end{equation*}

We also assume noise model $\xi \sim \mathcal{N}(0,1)$ and that for any $k \in [K]$, the values $$\mathbf{1}\left(f_k(a_1^{(1)}) > 1/2\right)  \cdots, \mathbf{1}\left(f_k( a_{\lceil \log(K)\rceil}   > 1/2 \right)$$ form the binary representation of $k$. This example is designed to ensure that allocating sufficient queries to actions in $\A_1$ will yield enough information to identify an $\epsilon <\epsilon_2$ optimal action (for any $\epsilon < \epsilon_2$) while at the same time generating linear regret.

Our first result is to ``sandwich'' the query complexity of this function class,

\begin{restatable}{lemma}{lemmasandwichingdepthrandomfeedback}\label{lemma::sandwiching_depth_random_feedback}
When $\epsilon_2 \leq \epsilon_1 $ and $K \geq 2$ the query complexity of $\F_K, \A_K$ satisfies, 
\begin{equation*}
    \depth^1_{0,1/4}(\F_K) \in \left[ \frac{\log(4/3)}{2\epsilon_1^2},\frac{16\log(K)\log\left(4\log(K) \right)}{\epsilon_1^2}\right]
\end{equation*}
\end{restatable}

\begin{proof}

As part of our argument we'll consider the ``empty'' problem $f_0$ over action sets $\A_1, \A_2$ defined as,
\begin{equation*}
    f_0(a) = \begin{cases}
            1/2 &\text{if } a \in \A_1\\
            1-\epsilon_2&\text{if } a \in \A_2
    \end{cases}
\end{equation*}     

Let $\Alg$ be an algorithm interacting with $\F_K,\A_K$ over $n$ steps and let's assume that $m_{\Alg}^1(0, 1/4) \leq n$ so that when interacting with any function in $\F_K$, after $n$ steps $\Alg$ can guess the correct optimal action with probability at least $3/4$.

Let's consider the interaction of $\Alg$ with function $f_0$. We use the notation $T_i^{(j)}(n)$ to denote the (random) number of queries $\Alg$ has performed on action $a_i^{(j)} \in \A_j$ for $j \in \{1,2\}$.

Let $\mathcal{E}_i$ for $i \in [K]$ denote the event that at time $n$ algorithm $\Alg$ outputs a guess $k(n)$ for the optimal action satisfying $k(n) = i$. So that,
\begin{equation*}
    \mathcal{E}_i = \{ k(n) = i\}.
\end{equation*}
Since $m_{\Alg}^1(0, 1/4) \leq n$ it follows that,
\begin{equation}\label{equation::lower_bound_correctness_prob_regret}
    \mathbb{P}_{\Alg, f_i}( \mathcal{E}_i) \geq 3/4.  
\end{equation}
Let's consider the problem $f_i$. The divergence decomposition Lemma~\ref{lemma::divergence_decomposition_bandits} implies,
\begin{align*}
    \mathrm{KL}( \mathbb{P}_{\Alg, f_0}  \parallel \mathbb{P}_{\Alg, f_i} ) &= \sum_{i' \in \left[ \log(K) \right]}\mathbb{E}_{\Alg, f_0}\left[  T_{i'}^{(1)}(n)  \right] \epsilon_1^2   +  \mathbb{E}_{\Alg, f_0}\left[  T_{i}^{(2)}(n)  \right] \epsilon_2^2 \\
    &\leq  n \epsilon_1^2 +  \mathbb{E}_{\Alg, f_0}\left[  T_{i}^{(2)}(n)  \right]  \epsilon_2^2 \\
    &\leq 2n\epsilon_1^2.
\end{align*}
The Huber-Bretagnolle inequality (Lemma~\ref{lemma::huber_bretagnolle_inequality}) applied to measures $\mathbb{P}_{\Alg, f_0}$ and $\mathbb{P}_{\Alg, f_i}$ for $i \in [K]$ implies,
\begin{equation*}
    \mathbb{P}_{\Alg, f_0}\left(  \mathcal{E}_i \right)  + \mathbb{P}_{\Alg, f_i}\left(  \mathcal{E}_i^c \right)  \geq \exp\left( - 2n\epsilon_1^2 \right)
\end{equation*}
summing over all $i \in [K]$ we get,
\begin{align*}
    K\exp\left( - 2n\epsilon_1^2 \right) &\leq \sum_i  \mathbb{P}_{\Alg, f_0}\left(  \mathcal{E}_i \right)  + \mathbb{P}_{\Alg, f_i}\left(  \mathcal{E}_i^c \right)  \\
    &\leq 1 + \sum_{i}  \mathbb{P}_{\Alg, f_i}\left(  \mathcal{E}_i^c \right)
\end{align*}
Thus, there is at least one index $\hat{i}$ such that,
\begin{equation*}
     \mathbb{P}_{\Alg, f_{\hat{i}}}\left(  \mathcal{E}^c_{\hat{i}} \right) \geq \exp\left( - 2n\epsilon_1^2 \right) - 1/K.
\end{equation*}
Equation~\ref{equation::lower_bound_correctness_prob_regret} implies the error probability $     \mathbb{P}_{\Alg, f_{\hat{i}}}\left(  \mathcal{E}^c_{\hat{i}} \right) \leq 1/4$ and therefore,
\begin{equation*}
    \exp\left( - 2n\epsilon_1^2 \right) - 1/K \leq 1/4.
\end{equation*}
Expanding the last inequality we obtain the condition $\exp\left( -2n\epsilon_1^2\right) \leq 3/4 $ and therefore $n \geq \frac{\log(4/3)}{2\epsilon_1^2}$. We conclude that,
\begin{equation*}
   \depth^1_{0,1/4}(\F_K) \geq  \frac{\log(4/3)}{2\epsilon_1^2} .
\end{equation*}

The upper bound $ \depth^1_{0,1/4}(\F_K) \leq \frac{16\log\left(4\log(K) \right)}{\epsilon_1^2}$ follows a simple argument. We exhibit an algorithm $\Alg$ that can find an optimal action after at most $\frac{16\log(K)\log\left(4\log(K) \right)}{\epsilon_1^2}$ queries. This procedure consists fo two parts,
\begin{enumerate}
    \item Explore all informative actions and estimate them up to a $\epsilon_1/4$ error. 
    \item Use these values to decode the location of the optimal action.
\end{enumerate}
Lemma~\ref{lemma::gaussian_single_concentration} implies Step 1 can be achieved by estimating the reward of each action in $\A_1$ up to $\epsilon_1/4$ error and probability at least $1-1/(4\log(K))$. This requires at most $16\log(4\log(K))/\epsilon_1^2$ queries for each of the $\log(K)$
actions in $\A_1$. A union bound implies this procedure is guaranteed to find the correct optimal action with a probability of error of at most $1/4$. We have therefore constructed an algorithm satisfying $m^1_{\Alg}(0, 1/4) \leq \frac{16\log(K)\log\left(4\log(K) \right)}{\epsilon_1^2}$. Thus,
\begin{equation*}
    \depth_{0,1/4}^1(\F_K) \leq \frac{16\log(K)\log\left(4\log(K) \right)}{\epsilon_1^2}.
\end{equation*}

\end{proof}

Our next supporting result shows there exists function classes such that any algorithm able to find an optimal action must incur in large regret. 

\begin{restatable}{lemma}{lemmaregretlowerboundsupportlemma}\label{lemma::regret_lower_bound_support_lemma}
    Let $T \in \mathbb{N}$ and $\epsilon_1, \epsilon_2  \in [0,1]$ such that $ \epsilon_1 \leq 1/4$ and $\epsilon_2 = 4\epsilon_1^2$. For any algorithm $\Alg$ such that $m_{\Alg}^1(0, 1/4) \leq T$ there is a problem $f \in \F_K$ such that,
 \begin{equation*}
    \mathbb{E}_{\Alg}[\regret(T, f) ] \geq  \frac{1}{64 \epsilon_1^2} .
 \end{equation*}
 for some universal constant $c>0$.

\end{restatable}

\begin{proof}
We'll again consider an empty problem $f_0'$ over action sets $\A_1$ and $\A_2$ defined as,
\begin{equation*}
    f_0'(a) = \begin{cases}
            1/2 &\text{if } a \in \A_1\\
            1&\text{if } a \in \A_2
    \end{cases}
\end{equation*}  
Just as in the proof of Lemma~\ref{lemma::sandwiching_depth_random_feedback} let's consider the problem $f_i$ for $i \in [K]$ and $\Alg$'s interaction with $f_0'$ and $f_i$ up to a horizon of $T$. The divergence decomposition Lemma~\ref{lemma::divergence_decomposition_bandits} implies,
\begin{align}
    \mathrm{KL}( \mathbb{P}_{\Alg, f_i}  \parallel \mathbb{P}_{\Alg, f_0'} ) &= \sum_{i' \in \left[ \log(K) \right]}\mathbb{E}_{\Alg, f_i}\left[  T_{i'}^{(1)}(T)  \right] \epsilon_1^2   + \sum_{i' \neq i}  \mathbb{E}_{\Alg, f_i}\left[  T_{i'}^{(2)}(T)  \right] \epsilon_2^2 \\
    &=  \mathbb{E}_{\Alg, f_i}\left[  T_{\A_1}(T)  \right] \epsilon_1^2 +  \mathbb{E}_{\Alg, f_i}\left[  T_{-i}^{(2)}(T)  \right]  \epsilon_2^2 \label{equaiton::KL_upper_bound_identity} 
\end{align}
where we define $T_{\A_1}(T) =\sum_{i' \in \A_1} T_{i'}^{(1)}(T)$ as the number of queries from actions in $\A_1$ and $T_{-i}^{(2)}(T) = \sum_{ i' \in \A_2 \backslash \{ i\} } T_{i'}^{(2)}(T)$. Pinsker's inequality implies,
\begin{equation}\label{equation::KL_pinsker_lower_bound}
    2\left(   \mathbb{P}_{\Alg, f_0'}\left( \mathcal{E}_i \right) -   \mathbb{P}_{\Alg, f_i}\left( \mathcal{E}_i \right)    \right)^2 \leq   \mathrm{KL}( \mathbb{P}_{\Alg, f_i} \parallel  \mathbb{P}_{\Alg, f_0'}  )
\end{equation}

Where $\mathcal{E}_i$ for $i \in [K]$ denote the event that at time $T$ algorithm $\Alg$ outputs a guess $k(T)$ for the optimal action satisfying $k(T) = i$. So that,
\begin{equation*}
    \mathcal{E}_i = \{ k(T) = i\}.
\end{equation*}

Since $\sum_{i=1}^K \mathbb{P}_{\Alg, f_0'}(\mathcal{E}_i) = 1$, there exists an index $\hat{i}$ such that $\mathbb{P}_{\Alg, f_0'}( \mathcal{E}_{\hat{i}}) \leq \frac{1}{K} \leq 1/2 $.

Since $m_{\Alg}^1(0, 1/4) \leq T$,  it follows that $\mathbb{P}_{\Alg, f_{\hat{i}}}\left(  \mathcal{E}_{\hat{i}}^c\right) \leq \frac{1}{4}$ and therefore $\mathbb{P}_{\Alg, f_{\hat{i}}}\left(  \mathcal{E}_{\hat{i}}\right) \geq \frac{3}{4}$. Thus,
\begin{equation}\label{equation::lower_bounding_probability_dists}
    \left|\mathbb{P}_{\Alg, f_0'}\left( \mathcal{E}_{\hat{i}} \right) -   \mathbb{P}_{\Alg, f_i}\left( \mathcal{E}_{\hat{i}} \right) \right| \geq 1/4.
\end{equation}
Combining Equations~\ref{equaiton::KL_upper_bound_identity},~\ref{equation::KL_pinsker_lower_bound} and~\ref{equation::lower_bounding_probability_dists} we get,
\begin{align}\label{equation::lower_bounding_kl_expression}
   \frac{1}{8} \leq \mathrm{KL}( \mathbb{P}_{\Alg, f_{\hat{i}}}  \parallel \mathbb{P}_{\Alg, f_0'} ) = \mathbb{E}_{\Alg, f_{\hat{i}}}\left[  T_{\A_1}(T)  \right] \epsilon_1^2 +  \mathbb{E}_{\Alg, f_{\hat{i}}}\left[  T_{-\hat{i}}^{(2)}(T)  \right]  \epsilon_2^2 .
\end{align}
Thus, $\max\left(\mathbb{E}_{\Alg, f_{\hat{i}}}\left[  T_{\A_1}(T)  \right] \epsilon_1^2, \mathbb{E}_{\Alg, f_{\hat{i}}}\left[  T_{-\hat{i}}^{(2)}(T)  \right]  \epsilon_2^2  \right) \geq \frac{1}{16}$ so that at least one of the following two inequalities holds,
\begin{enumerate}
    \item[A)] $\mathbb{E}_{\Alg, f_{\hat{i}}}\left[  T_{\A_1}(T)  \right] \geq \frac{1}{16\epsilon_1^2} $ .
    \item[B)] $\mathbb{E}_{\Alg, f_{\hat{i}}}\left[  T_{-\hat{i}}^{(2)}(T)  \right] \geq \frac{1}{16\epsilon_2^2}$.
\end{enumerate}

Finally, 
\begin{equation*}
    \mathbb{E}_{\Alg, f_{\hat{i}}}\left[\regret( T, f_{\hat{i}})\right] \geq \frac{1}{4} \mathbb{E}_{\Alg, f_{\hat{i}}}\left[  T_{\A_1}(T)  \right] + \mathbb{E}_{\Alg, f_{\hat{i}}}\left[  T_{-\hat{i}}^{(2)}(T)  \right]  \epsilon_2.
\end{equation*}

And therefore when case $A)$ holds,
\begin{equation*}
    \mathbb{E}_{\Alg, f_{\hat{i}}}\left[\regret( T, f_{\hat{i}})\right] \geq \frac{1}{64 \epsilon_1^2}.
\end{equation*}

When case $B)$ holds, 

\begin{equation*}
    \mathbb{E}_{\Alg, f_{\hat{i}}}\left[\regret( T, f_{\hat{i}})\right] \geq \frac{1}{16 \epsilon_2} \stackrel{(i)}{\geq} \frac{1}{64 \epsilon_1^2}.
\end{equation*}
Inequality $(i)$ follows because $\epsilon_2 = 4\epsilon_1^2$.

\end{proof}

Finally, we are ready to prove Theorem~\ref{theorem::main_theorem_result_noisy_separation}.

\theoremmainresultnoisyseparation*

\begin{proof}
To prove this Theorem we leverage our supporting results from Lemma~\ref{lemma::sandwiching_depth_random_feedback} and~\ref{lemma::regret_lower_bound_support_lemma}. We consider  $(\F_K, \A_K)$ for $K = 2$. Under this choice Lemma~\ref{lemma::sandwiching_depth_random_feedback} implies that,
\begin{align*}
  \depth^1_{0,1/4}(\F_K) \in \left[ \frac{\log(4/3)}{2\epsilon_1^2},\frac{16\log(K)\log\left(4\log(K) \right)}{\epsilon_1^2}\right]   &= \left[ \frac{\log(4/3)}{2\epsilon_1^2},\frac{16\log(2)\log\left(4\log(2) \right)}{\epsilon_1^2}\right] 
\end{align*}

Setting $d = \frac{\log(4/3)}{2\epsilon_1^2}$ so that $\epsilon_1 = \sqrt{\frac{\log(4/3)}{2d}}\leq \frac{1}{\sqrt{d}}$ and noting that $16\log(2)\log\left(4\log(2) \right)*2/\log(4/3) < 80$ we get,
\begin{align*}
     \depth^1_{0,1/4}(\F_K) \in [d, 80d].
\end{align*}

Moreover, Lemma~\ref{lemma::regret_lower_bound_support_lemma} implies that for any algorithm $\Alg$ such that $m_{\Alg}^1(0, 1/4) \leq T$ there is a function $f \in \F_K$ such that,
\begin{equation*}
\max_{f \in \F_K}    \mathbb{E}_{\Alg}[\regret(T, f) ] \geq  \frac{1}{128\epsilon_1^2} \geq \frac{d}{128}.
 \end{equation*}
To finalize the proof we note that UCB satisfies an expected regret bound of  order $8\sqrt{2T\log(T)}$ when using only actions in $\A_2$.

This finalizes the proof.
\end{proof}

\section{Useful Lemmas}

\begin{lemma}[Supporting Result]\label{lemma::gaussian_single_concentration}
Let $\delta \in (0,1)$, $\sigma > 0$ and $\xi \sim \mathcal{N}(0,\sigma^2)$.  Then, the probability that, 
\begin{equation*}
    \mathbb{P}\left( | \xi | \geq \sigma\sqrt{2\log(2/\delta) }  \right) \leq \delta.
\end{equation*}
In addition, let $\hat{\xi} = \frac{1}{n} \sum_{i=1}^n \xi_i$ be the empirical average of $n$ independent samples from $\mathcal{N}(0,\sigma^2)$ then,
\begin{equation*}
    \mathbb{P}\left( \left|\hat{\xi}\right| \geq \sigma \sqrt{ \frac{2\log(2/\delta)}{n}   }   \right) \leq \delta
\end{equation*}
\end{lemma}

\begin{proof}
The random variable $\xi$ is $\sigma^2$-subgaussian. Subgaussian random variables satisfy,
\begin{equation*}
    \mathbb{P}\left( | \xi | \geq \sigma \sqrt{2 \log(2/\delta)} \right) \leq \delta.
\end{equation*}
the second result follows because the subgaussian parameter of $\hat{\xi} $ is $\frac{\sigma^2}{n}$. The result follows. 
\end{proof}

 \begin{lemma}\label{lemma::supporting_lemma_min_sequences}
Let $a_1,  \cdots, a_K $ and $ b_1, \cdots, b_K  $ be two sequences of nonnegative numbers. There exists an index $\tilde{i}$ such that,
\begin{align*}
 a_{\tilde{i}}   \leq \frac{3\sum_{i=1}^K a_i}{K}, \qquad
  b_{\tilde{i}}   \leq \frac{3\sum_{i=1}^K b_i}{K}
\end{align*}

\end{lemma}

\begin{proof}
For simplicity let $A = \sum_{i=1}^K a_i$ and $B = \sum_{i=1}^K b_i$. Let $\sigma(1), \cdots, \sigma(K)$ be the permutation permutation of $[K]$ such that $a_{\sigma(1)} \leq a_{\sigma(2)} \leq \cdots  \leq a_{\sigma(K)}$. Each of $a_{\sigma(i)}$ for $i = 1, \cdots, \lfloor K/2 \rfloor$ satisfies,
\begin{equation*}
    a_{\sigma(i)} \leq 2A/K
\end{equation*}
This is because if this was not true then, 
$a_{\sigma(i)} > 2A/K$ for all $i \geq \lfloor K/2 \rfloor$ which would imply $$\sum_{i=1}^K a_i \geq \sum_{i=\lfloor K/2\rfloor +1}^{K} a_{\sigma(i)}> (K-\lfloor K/2\rfloor  ) \cdot 2A/K \geq A,$$ a contradiction.
Consider now values $b_{\sigma(1)}, \cdots, b_{\sigma(\lfloor K/2 \rfloor)}$. The sum of these $\lfloor K/2 \rfloor$ values satisfies $$\sum_{i=1}^{\lfloor K/2 \rfloor} b_{\sigma(i)} \leq \sum_{i=1}^K b_{\sigma(i)} = B.$$ Thus, it follows there is an index $\bar{i} \in \lfloor K/2 \rfloor$ such that,
\begin{equation*}
    b_{\sigma(\bar{i})} \leq \frac{B}{\lfloor K/2 \rfloor} \leq 3B/K. 
\end{equation*}
Since $a_{\sigma(\bar{i})} \leq 2A/K < 3B/K$ the result follows by setting $\tilde{i} = \sigma(\bar{i)}$. 
\end{proof}

\begin{lemma}[Huber Bretagnolle Inequality]\label{lemma::huber_bretagnolle_inequality}
Let $P, Q$ be two measures and $\mathcal{E}$ be a measurable event. Then, 
\begin{equation*}
P(\mathcal{E}) + Q(\mathcal{E}^c) \geq \exp\left(-\KL(P\parallel Q) \right) 
\end{equation*}

\end{lemma}

As well as the divergence decomposition for Bandit problems,
\begin{lemma}\label{lemma::divergence_decomposition_bandits}
Let $\nu = (P_1, \cdots, P_k)$ be the reward distributions associated with one $k$-armed bandit, and let $\nu' = (P_1', \cdots, P_k')$ be the reward distributions associated with another $k$-armed bandit. Fix some policy $\pi$ and let $\mathbb{P}_\nu = \mathbb{P}_{\nu \pi}$ and $\mathbb{P}_{\nu'} = \mathbb{P}_{\nu' \pi}$ be the probability measures on the canonical bandit model induced by the $n$-round interconnection of $\pi$ and $\nu$ (respectively $\pi$ and $\nu'$). Then  
\begin{equation*}
    \mathrm{KL}(\mathbb{P}_\nu \parallel \mathbb{P}_{\nu'}) = \sum_{i=1}^k \mathbb{E}_{\nu}[T_i(n)] \mathrm{KL}(P_i \parallel P_i')
\end{equation*}
where $T_i(n)$ is the random variable specifying how many times $\pi$ selects action $i$ up to round $n$.

\end{lemma}
See for example Theorem~14.2 and Lemma~15.1 in~\cite{lattimore2020bandit} for a reference of Lemmas~\ref{lemma::huber_bretagnolle_inequality} and~\ref{lemma::divergence_decomposition_bandits}.

\end{document}